\documentclass[final,onefignum,onetabnum]{siamonline190516}

\newcommand{\R}{\mathbb{R}}
\usepackage{graphicx}
\usepackage{verbatim}

\usepackage{amsmath}
\usepackage{multirow}
\usepackage[autostyle]{csquotes}
\usepackage{amssymb}

\usepackage{lipsum}
\usepackage{amsfonts}
\usepackage{graphicx}
\usepackage{epstopdf}
\usepackage{algorithmic}
\ifpdf
  \DeclareGraphicsExtensions{.eps,.pdf,.png,.jpg}
\else
  \DeclareGraphicsExtensions{.eps}
\fi

\usepackage{enumitem}
\setlist[enumerate]{leftmargin=.5in}
\setlist[itemize]{leftmargin=.5in}

\newsiamremark{remark}{Remark}
\newsiamremark{hypothesis}{Hypothesis}
\crefname{hypothesis}{Hypothesis}{Hypotheses}
\newsiamthm{claim}{Claim}

\headers{Hierarchical Data Reduction and Learning}{P.Shekhar and A.Patra}

\title{Hierarchical Data Reduction and Learning\thanks{Submitted to the editors October 21, 2019.
\funding{This work was funded by the grants NSF1821311, NSF1645053, NSF1621853.}}}

\author{Prashant Shekhar\thanks{Data Intensive Studies Center, Tufts University 
  (\email{prashant.shekhar@tufts.edu}).}
\and Abani Patra\thanks{Data Intensive Studies Center, Department of Computer Science and Department of Mathematics, Tufts University 
  (\email{abani.patra@tufts.edu}).}}

\usepackage{amsopn}

\ifpdf
\hypersetup{
  pdftitle={Hierarchical Data Reduction and Learning},
  pdfauthor={Prashant Shekhar, Abani Patra}
}
\fi

\begin{document}

\maketitle

\begin{abstract}
This paper describes a hierarchical learning strategy for generating sparse representations of multivariate datasets. The hierarchy arises from approximation spaces considered at successively finer scales. A detailed analysis of stability, convergence and behavior of error functionals associated with the approximations are presented, along with a well chosen set of applications. Results show the performance of the approach as a data reduction mechanism for both synthetic (univariate and multivariate) and real datasets (geospatial and numerical model outcomes). The sparse representation generated  is  shown to efficiently reconstruct  data and minimize error in prediction.
\end{abstract}



\section{Introduction}
Hierarchical and multiscale modeling is widely used both in  data driven and physics based models \cite{vogelstein2014discovery,gavish2010multiscale,
koumoutsakos2005multiscale,quarteroni2003analysis}. In this paper, we focus on a strategy which analyzes  data at different scales ($s$) by constructing corresponding basis representations $B^s \in \R^{n \times l_s}$ and inferring if these scale and data dependent bases  are able to approximate the observed data $f|_X \in \R^n$ ($X \in \R^{n \times d}$) to satisfy: 
\begin{equation}\label{sys}
\begin{aligned}
& \underset{C_s \in \R^{l_s}}{\text{min}}
||f|_X -B^s C_s||_2
\leq TOL
\end{aligned}
\end{equation}
where, $l_s \in \mathbb{N}$ is a scale enumeration chosen by the algorithm, $TOL$ is a user defined tolerance level, and $n$ is the number of observations in a $d$ dimensional space ($x_i \in \R^d$, $1 \leq i \leq n$). $C_s$ in (\ref{sys}) is computed  as the optimal coordinate of projection on $B^s$. Besides constructing these bases at each scale, the algorithm identifies \enquote{representative} data points which form the corresponding sparse data $X_s$. $\{s, X_s, C_s\}$ allow reconstruction of the original dataset and prediction (model) at any new point in the domain of interest resulting in both  a sparse representation \cite{lane1999temporal,chou2006generalized,shih2003text}  and an efficient model  \cite{fricker2013multivariate,carr1997surface}.

Data sparsification seeks to find a subset  $X_s \subseteq X=\{x_1,x_2,..,x_n\} \subset \Omega$ ($x_i \in \R^d$ for $d$ dimensional inputs) for which we have data $f|_{X} = {(f(x_1), f(x_2),...,f(x_n))}^T \in \R^n$, such that, we may obtain an acceptable reconstruction or approximation to $f$ as in equation (\ref{sys}). This subset with the corresponding projection coordinate $C_{s}$ with respect to the basis set $B^s$, is regarded as our obtained sparse representation ($D^s_{sparse} = [X_s, C_s]$). Thus, this compressed representation of the dataset acts as a model for new prediction points $X_{*} \subset \Omega$ and reconstruction of the original dataset $f|_{X}$ at $X \subset \Omega$ as needed. 
The approximation problem \cite{cheney1966introduction,powell1981approximation} is a more traditional one seeking to construct an approximation $(Af): \R^d \to \R$ for the underlying function $f$ which was discretely observed ($f|_{X}$) such that it minimizes the normed error measure $||(Af)  - f||$. Following the standard notation, for a normed space $\mathcal{V}$ and any approximation $\tilde{f}$ to $f$ such that $f, \tilde{f} \in \mathcal{V}$, the quality of approximation $\tilde{f}$ is usually quantified by the error norm ${||f - \tilde{f}||}_{\mathcal{V}}$. $\tilde{f}$ here is usually constructed by first fixing a finite dimension approximation space $\mathcal{Q} \subseteq \mathcal{V}$ and then computing the optimal approximation $(Af) \in \mathcal{Q}$ by minimizing normed distance between $f$ and $\mathcal{Q}$:

\begin{equation}\label{Q}
\underset{\tilde{f} \in \mathcal{Q}}{\text{min}}
||f - \tilde{f}||_{\mathcal{V}}
=
||f - (Af)||_{\mathcal{V}}
\end{equation}

We now motivate our development of   a data dependent bases ($B^s \equiv B(X)^s$) \cite{de2010stability}. 
The Mairhuber-Curtis theorem \footnote{Theorem 2 in \cite{iske2011scattered}} states that for $d \geq 2$ and $n \geq 2$, there is no Chebyshev system $B = \{s_1,...s_n\}$ on $\R^d$, therefore for some data $X = \{x_1,x_2,..,x_n\} \subset \R^d$, the Vandermonde matrix approaches singular behavior affecting the quality of approximation. As a remedy, we choose a data dependent continuous kernel function $K: \R^d \times \R^d \to \R$ which is positive definite and symmetric on $\Omega$. Thus, for data points $X = \{x_1, x_2,....,x_n\}$, the kernel function becomes $K(x_i,x_j)_{1 \leq i,j \leq n}$. From the literature on these functions \cite{aronszajn1950theory,buhmann2003radial,wendland2004scattered}  these kernels are reproducing in the native Hilbert Space $\mathcal{H}$ of functions on $\Omega$ in the sense

\begin{equation}
<f,K(x,\cdot)>_\mathcal{H} = f(x) \quad x \in \Omega, f \in \mathcal{H}
\end{equation}

The functionals $K(x,.)$ are the Reisz representers of the linear point functionals $\delta_x: f \mapsto f(x)$ in the dual space  $\mathcal{H}^*$ of  $\mathcal{H}$. Therefore we have the relationship

\begin{equation}
K(x,y) = <K(x,\cdot),K(y,\cdot)>_ \mathcal{H} = <\delta_x,\delta_y>_ {\mathcal{H}^*}  \quad x,y \in \Omega
\end{equation}

For observed data at $X = \{x_1,x_2,...,x_n\}$, one can consider the space of trial function $\mathcal{D} _X = span\{K(\cdot,x_j): x_j \in X\}$ as a choice for the approximation space. However from work such as \cite{de2010stability,fasshauer2009preconditioning} we know that the bases formed by translates of the kernel functions are highly dependent on the distribution of data points in X and hence numerically ill-conditioned for many problems. For dealing with this issue, at each scale s, our approach considers a subset of original trial functions as bases such that

\begin{equation}\label{gamma}
\Gamma^s = span\{K^s(\cdot,x_i): x_i \in X_s\} \approx span\{K^s(\cdot,x_j): x_j \in X\}, \quad X_s \subseteq X
\end{equation}

This is equivalent to identifying a lower dimension manifold in the space $\mathcal{D}_X$ (as compared to original dimension   n for n data points, represented by the kernel matrix $K$). These linearly independent data derived basis functions are sampled using a Pivoted QR strategy as detailed in the following section (Algorithm \ref{multiscale_algo}). It should be noted here that at each scale we have a different kernel function $K^s$. Hence, effectively we are working in a different RKHS at every scale (from the uniqueness property \cite{aronszajn1950theory} of kernel functions). Let this scale dependent RKHS be represented as $\mathcal{H}_s$. Therefore at any scale s, the function to be approximated ($f$) is assumed to be in the RKHS $\mathcal{H}_s$. However, our approach is aimed at finding an approximation $A_sf \in \Gamma^s$ to $f$ as in (\ref{gamma}). Hence, if we look at the problem formulation (\ref{Q}), then $\mathcal{H}_s$ is the space $\mathcal{V}$ and $\Gamma^s$ is our approximation space $\mathcal{Q}$. 
While our overall approaches  are somewhat related to  statistical emulators \cite{gu2016parallel}, the notion of scale is rarely articulated in that domain.

The contributions of this paper can be summarized as follows. Firstly we introduce an approach which computes a scale ($s$) dependent sparse representation $D^s_{sparse}$ of a large dataset. Besides data reduction, the algorithm also provides a model for approximations at new data points in the observation space using just $D^s_{sparse}$, enabling any sort of learning from a compressed version of the dataset. We then provide detailed analysis on convergence, establishing the optimality of the solution obtained by the algorithm, computation of point-wise error functionals and their behavior which governs the performance of the approach. At each scale, besides providing confidence bounds which quantify our belief in the estimation, we also provide prediction intervals which estimates the bounds in which the algorithm believes any observations not included in the sparse representation as well as any new samples to be made in the future will lie. This also makes the sparse representation more useful. Towards the end, we have also provided  stability estimates to further establish the dependability of the approach.

\section{The Hierarchical Approach}
\label{sec:mult}
In this approach, developing further the work of \cite{bermanis2013multiscale}, we introduce a methodology for data reduction through efficient bases construction exploiting the multilevel nature of the correlation structure present in the data. The basic steps involved in the approach are presented as Algorithm  \ref{multiscale_algo}, here.
\begin{algorithm}
\caption{Hierarchical approach}\label{multiscale_algo}
\begin{algorithmic}[1]
\STATE{\textbf{INPUT}:\\
\quad \textit{Hyperparameters}: $[(TOL \geq 0, T>0, P > 1) \in \R^3]$\\
\quad \textit{Dataset}: $[(X \in \R^{n \times d}) \subset \Omega$, $(f|_{X} \in \R^n)]$ \\
\quad \textit{Prediction points}: $[(X_{*} \in \R^{n^{*} \times d}) \subset \Omega]$}
\STATE{\textbf{OUTPUT}:\\ 
\quad \textit{Convergence Scale}: $[S_a \in \mathbb{N}]$\\
\quad \textit{Sparse Representation ($D_{sparse}^{S_a}$)}: $[(X_{S_a} \in \R^{\{l_{S_a} \times d \}}),(C_{S_a} \in \R^{l_{S_a}})]$\\
\quad \textit{Predictions}: $[P^* \in \R^{n^{*}}]$} \\
\noindent\rule{14.0cm}{0.4pt}
\STATE{Initialize: $s =0$}
\WHILE{{$TRUE$}}
\STATE{Compute covariance kernel: [$G_s$ on $X$ with $(\epsilon_s = T/P^s$)]}
\STATE{Compute numerical Rank: [$l_s = rank(G_s)$]}
\STATE{Remove sampling Bias: [$(W = AG_{s})$ with $A \in \R^{k \times n}$ and $(a_{i,j} \sim N(0,1))$]}
\STATE{Generate permutation information: [$WP_R = QR$]}
\STATE{Produce bases at scale s: $[B^{s} = (G^{s}P_R)[:,1:l_s]$, \quad \text{with}\  $B^{s} \in \R^{n \times l_s}]$}
\STATE{Subset the sparse representation in $X_s$}
\STATE{Compute coordinate of projection: [$C_{s} = {(B^{s})}^{\dagger} f$ where  ${B^s}^{\dagger} = ({B^s}^TB^s)^{-1}{B^s}^T$}]
\STATE{Generate approximation at s: [$(A_sf)|_{X_s} = B^{s}C_s$]}
\STATE{Update for next scale: ($F^{s} = (A_sf)|_{X_s}; s=s+1$)}
\STATE{If $(||f|_X- F^{s-1}||_2 \leq TOL)$ : $S_a = s-1$; $Break$}
\ENDWHILE
\STATE{Compute bases for prediction at $X_{*}$: [$G^{*}_{S_a}$ centered at $X_{S_a}$with $(\epsilon_{S_a} = T/P^{S_a}$)]}
\STATE{Predict: $P^* = G^{*}_{S_a} C_{S_a}$}
\end{algorithmic}
\end{algorithm}

\subsection{Proposed Algorithm}The proposed approach (Algorithm \ref{multiscale_algo}) constructs a sequence of scale (s) dependent approximations $A_1f, A_2f,..,A_sf..$ to the unknown function $f: \R^{d} \mapsto \R$. Each of these approximations only use a subset of dataset $X_1, X_2,...,X_s,..$ respectively. Algorithm begins by taking a dataset, where a data point $x_i \in \R^d$ is mapped to a functional value $f(x_i) \in \R$.  In matrix form $f|_{X} = \{f_1,f_2,....f_n\}$ values are obtained at data points $X = \{x_1, x_2,....,x_n\}$ ($f|_{X} \in \R^n$ and $X \in \R^{n \times d}$). The scalars $[T, P, TOL]\in \R^3$ are the algorithmic hyperparameters defined by the user. TOL is the simple $2$-$norm$ error tolerance, P is assumed to be 2 (Based on \cite{bermanis2013multiscale}). This choice of P reduces the length scale of the kernel (K) by a factor of 0.5 at each scale increment, providing an intuitive understanding of how the support of basis functions is adapted to scale variation. If we assume the diameter of the dataset to be distance between the most distant pair of datapoints, then T is given by

\begin{equation}
T = 2 (Diameter(X)/2)^2
\end{equation}

Besides these hyperparameters, the algorithm also accepts the prediction points $X_{*} \subset \Omega \in \R^{n^{*} \times d} $, which represent the data points at which the user wants to approximate the underlying function. We also need to choose the positive definite function ($K: \R^d \times \R^d \to \R$). For this paper we use the squared exponential (also referred to as Gaussian) kernel (\ref{sqexp}) for mapping the covariance structure and generating the space of trial functions $\Gamma^s$(equation \ref{gamma}) at each scale s. 

\begin{equation}
G_s (x,y) = \exp\left({-\frac{||x-y ||^2}{\epsilon_s}}\right) \mbox{ ;  }  \epsilon_s=\frac{T}{P^s} 
\label{sqexp}
\end{equation}

Here $\epsilon_s$ (also known as the length scale parameter) determines the width of the basis functions at a particular scale. This squared exponential kernel is very widely used in the Gaussian process literature \cite{rasmussen2004gaussian} and is a popular choice for most spatial statistics problems. The learning phase (for generating $D^s_{sparse}$) of the proposed algorithm has been explained in $STEP-1$ and $STEP-2$ below. $STEP-3$ constitutes the prediction phase which uses the produced sparse representation for the dataset. It should be noted that, it is not required to wait till the convergence of the algorithm to go to the prediction phase. In fact, it is possible to predict from the sparse representation at each scale ($D^s_{sparse}$). We have shown the prediction phase separately for the purpose of clarity.

\textbf{STEP-1 (Getting sparse data-$X_s$)}: Given the Dataset $D = [X, f|_X]$, the algorithm begins with the computation of the covariance operator $G_s$ (\ref{sqexp}). However, based on \cite{de2010stability,fasshauer2009preconditioning}, the distribution of the dataset might lead to ill-conditioning of this covariance kernel. Therefore we carry out a column pivoted QR decomposition to identify the space $\Gamma^s$ (at each scale) which represents the span of the trial functions $K^s(\cdot, x_j), (1 \leq j \leq n)$ at scale s. The QR decomposition is carried out on W for obtaining the Permutation matrix $P_R$. $W$ is produced by the product of a random normal matrix $A$ with the $G_s$. Here we have $A \in \R^{k \times n}$ with $l_s = rank(G_s) \leq k \leq n$. For our experiments we have assumed $k = l_s+8$ (as in \cite{bermanis2013multiscale}) which means we sample 8 additional rows to account for numerical round-offs during the QR decomposition. However, this is a conservative step and even without any additional sampled columns, the algorithm was found to perform well. The permutation matrix $P_R$ produced by the decomposition captures the  information content of each column of W. $P_R$ is then used to extract independent columns with the biggest norm contributions along with the observation points these columns correspond to in the covariance kernel ($G_s$). This set of sampled observations from the original dataset is termed as the corresponding sparse data ($X_s$). The number of columns sampled from $G_s$ come from its numerical rank estimated by using strategies such as a $Rank\ Revealing-QR$ or a $SVD$ decomposition.

\textbf{STEP-2 (Getting projection coordinate-$C_s$)}:  Once, we have the relevant columns of $G^s$ (denoted as $B^s$) which also represent the approximation subspace $\Gamma^s$ (in the native RKHS) which spans $\mathcal{D}_X$, the algorithm proceeds to solve the over-constrained system ($B^{s}C_s = f|_X$). We can think of it as an orthogonal projection problem where $f|_X$ needs to be projected on the column space of $B^{s}$ and then it is required to compute the specific weighting of vectors in the basis matrix $B^{s}$ which produces this projection. The Algorithm computes the orthogonal projection ($(A_sf)|_X$) given the required coordinates and basis vectors in $B^s$.
Once we have the scale $s = S_a$ at which the algorithm satisfies the $2$-$norm$ condition, it produces the sparse representation $D^{S_a}_{sparse} = \{X_{S_a}, C_{{S_a}}\}$ and proceeds to the prediction stage if required by the user. Here $S_a$ is called the convergence scale as detailed in the following subsection.

\textbf{STEP-3 (Prediction at $X_*$ from $D^{S_a}_{sparse}$)}: For computing functional values at unobserved location $X_*$, basis functions are constructed by computing $G^*_{S_a}$ (Gaussian kernel for the prediction points with the sparse data - $X_{S_a}$). The Prediction step weighs the constructed bases with coordinates of orthogonal projection $C_s$ ($C_{S_a}$, if the prediction is being made at the convergence scale) and linearly combines them to produce the required approximation.

Finally, before moving forward, it is worth mentioning here that the $2$-$norm$ criteria is just one of the many possible kinds of norms which can be used to measure the scale dependent fidelity of the model.

\subsection{Critical Scale ($S_c$) and Convergence Scale ($S_a$)}
In our work, the scale at which the kernel matrix becomes  numerically full rank is referred to as Critical Scale. Working with proposition 3.7 in \cite{bermanis2013multiscale}, if $\delta$ represents the precision of rank for the Gaussian kernel matrix, then we can define the numerical rank of the kernel as

\begin{equation}
 R_\delta (G_s) = \#\Bigg(j: \frac{\sigma_j (G_s)}{\sigma_0 (G_s)} \geq \delta \Bigg)
\end{equation}

where $\sigma_j (G_s)$ is the $j^{th}$ largest singular value of $G_s$. Also let $|I_i|$ represents the length of the bounding box of the data in $i^{th}$ ($i \in [1,d]$) dimension. Then given the length scale parameter $\epsilon_s$, the rank of the Gaussian kernel can be bounded above as 

\begin{equation}
R_\delta (G_s) \leq \prod_{i = 1}^d \Bigg( \frac{2 |I_i|}{\pi} \sqrt{\epsilon_s^{-1}ln(\delta^{-1})} + 1 \Bigg)
\end{equation} 

Proposition 3.7 in \cite{bermanis2013multiscale} states that numerical rank of the Gaussian kernel matrix is proportional to the volume of the minimum bounding box $B = I_1 \times I_2 \times ....\times I_d$ and to  $\epsilon^{-d/2}$ . Therefore for a fixed data distribution, following relation holds

\begin{equation}\label{r_del}
R_\delta (G_s) \propto \epsilon^{-d/2}  \propto P^{sd} \propto s; \quad \text{for }P=2
\end{equation}

Hence numerical rank of the Gaussian covariance kernel is directly proportional to the scale of study. Therefore, there exists a minimum scale $s = S_c$, at which the kernel becomes full and continue to stay full rank as the scale is further increased.This scale is referred to as the Critical Scale ($S_c$). Hence, based on the overall Algorithm \ref{multiscale_algo}, If $X_1, X_1..X_s..$ and so on, are the sampled sparse representation at scale s, they satisfy 
$|X_1|  \leq |X_2|  \leq |X_3|  \leq ..... \leq  |X_{S_c}| $
where $|\cdot|$ denotes the cardinality of a set. Therefore with increasing scales  more data points are added leading to better approximation. Thus, for any scales $j$ and $i$ satisfying the relation $j \geq i$, the following relation holds

\begin{equation}
||f|_{X} -(A_{{j}}f)|_{X}||_{2} \leq ||f|_{X} - (A_{i}f)|_{X} ||_2
\end{equation}

From approximation theory \cite{cheney2009course}, we can now state that for  kernel at the Critical Scale and a finite point set $X \subset \Omega$: (a) $K^{S_c}$ is  positive definite, (b) If $(A_{{S_c}}f) \in \Gamma^{S_c}$ vanishes on $X$, then $(A_{{S_c}}f) \equiv 0$, and, c) the approximation $(A_{{S_c}}f)$ satisfying the following relation is unique
\[
\lim_{s \to S_c}(A_{{s}}f)|_{X} = f|_{X}
\]

%
%

We can now define the Convergence scale ($S_a$)  as the minimum scale $s$ which satisfies 
\[
||f|_{X} - (A_{s}f)|_{X}||_2
\leq TOL
\]
Therefore it is the scale at which Algorithm  \ref{multiscale_algo}, stops. It is worth noting here that based on the definition of $S_c$ and $S_a$, we have $S_a \leq S_c$ and $|X_{S_a}| \leq  |X_{S_c}|$ 

\section{Convergence Properties}
 Following the notations introduced in the previous sections, again let $f$ be the function which needs to be approximated and $A_{s}f$ be the approximation produced by the Algorithm \ref{multiscale_algo} at scale s. Correspondingly, let $E^s = f - A_sf$ be the error in approximation at scale $s$ with it's $2$-$norm$  in $\R^n$ as
$ ||E^s|_X||_2 = ||f|_{X} - (A_{s}f)|_{X}||_2$.

\begin{lemma}
For any choice of $TOL \geq 0$  and any norm: $||\cdot|| \in \R^n$ chosen by the user, the proposed algorithm (Algorithm \ref{multiscale_algo}) converges in the sense

\begin{equation}
||E^s|_X|| = ||f|_X - (A_sf)|_X|| \leq TOL
\end{equation}

for some finite $s \leq S_c$ and the corresponding produced approximation $A_sf$
\end{lemma}
\begin{proof}
$(A_{s}f)|_X$  is the projection of  $f|_X$ on the bases at scale $s$ with the projection operator  $B^s {B^s}^{\dagger} \equiv R^s$, where ${B^s}^{\dagger}$ is the pseudo-inverse of basis matrix $B^s$. 
Therefore
\begin{equation}\label{Err}
\begin{split}
||E^s|_X||_2 &  =  ||f|_{X}-R^sf|_{X}||_2
 = ||(I-R^s)f|_{X}||_2 \leq ||I-R^s||_2|| f|_{X}||_2\\
\end{split}
\end{equation}

From Equation (\ref{r_del}), that numerical rank of bases $B^s$ increases monotonically with scale $s$. Hence  $s \to S_c \Rightarrow
R^s \to I_n$.
Thus $||I-R^s||_2 \to 0 $ and hence $||E^s|_X||_2 \to 0$. 
%
\begin{equation}\label{conv}
\lim_{s \to S_c} ||E^s|_X||_2 \to 0
\end{equation}

This establishes the convergence for the proposed approach in $2$-$norm$ at some finite scale $S_a \leq S_c$. 
%
%
%
%
From equivalence of norms on finite dimensional vector space  \cite{kreyszig1978introductory} this  generalizes the convergence guarantee to all possible norm in $\R^n$. 
%
%
\end{proof}

In the following theorem we analyze a particular hierarchical updating scheme.

\begin{theorem}\label{thm:bigthm}
If the projection update of Algorithm \ref{multiscale_algo} is written in an iterative form
\begin{equation}
(A_{{s+1}}f)|_X = (A_{s}f)|_X + \alpha_s \cdot E^s|_X
\end{equation}
where $E^s = f - A_sf$. Then $\alpha_s$ follows the bounds 
%
$0 <  \alpha_s < 2$
%
and convergence rate ($\rho_s$) satisfies the relation
$\rho_s  = |1-\alpha_s|$
%
\end{theorem}

\begin{proof}

Since at each scale s, $A_{s}f$ is generated as projection of $f$ on space

\[\Gamma^s = span\{K^s(.,x_i): x_i \in X_s\} \quad X_s \subseteq X\]
Therefore using the notation $E^s = f - A_sf$, we write this update in an iterative form
\begin{equation} \label{update}
(A_{{s+1}}f)|_X = (A_{s}f)|_X + \alpha_s \cdot E^s|_X
\end{equation}

From the definition of $E^s|_X$ it follows that any convergent algorithm requires $\alpha_s \geq 0$.
%
Rearranging (\ref{update})
%
%
and taking an inner product with respect to $E^s|_X$ 
\[
\alpha_s = \frac{<(A_{{s+1}}f - A_{s}f)|_X, E^s|_X>}{<E^s|_X, E^s|_X>}= \frac{<(A_{{s+1}}f - A_{s}f)|_X, E^s|_X>}{||E^s|_X||^2}
\]
Using the Cauchy Schwarz and Triangle inequalities
\begin{equation*}
\alpha_s  \leq \frac{||(A_{{s+1}}f - A_{s}f)|_X|| \cdot ||E^s|_X||}{||E^s|_X||^2} 
= \frac{||(A_{{s+1}}f - A_{s}f)|_X|| }{||E^s|_X||} \leq  1 + \frac{||E^{s+1}|_X||}{||E^s|_X||}
\end{equation*}
Therefore
\begin{equation}
 0 \leq \alpha_s \leq  1 + \frac{||E^{s+1}|_X|| }{||E^s|_X||}
\end{equation}
Now, for convergent approximations  $||E^{s+1}|_X|| \leq ||E^s|_X||$. Therefore for all scales, $\alpha_s \in [0,2]$. The end members $\alpha_s = \{0,2\}$ imply the iterative scheme has converged and no improvement is needed (as reflected in the stopping criterion in Algorithm \ref{multiscale_algo}). Hence, we remove these end members obtaining our desired bounds.
Rewriting  equation (\ref{update})
\[
(A_{{s+1}}f -f)|_X = (A_{s}f -f)|_X + \alpha_s(f-A_{s}f)|_X \text { leads to}
\]
\begin{equation}
\rho_s =\frac{||(f - A_{{s+1}}f)|_X|| }{||(f-A_{s}f)|_X||} = |1 - \alpha_s|
\end{equation}
\end{proof}

Now we analyze the inner product of the error of the projection ($E^s = f - A_sf$) with respect to projection at scale $s$ in the native RKHS. The major power of the following theorem lies in the fact that as the proposed algorithm converges, we are able to upper bound this inner product as a function of the user defined tolerance ($TOL$). This provides the user direct control on the approximation properties on the algorithm even in the native RKHS. It should be noted again that at each scale $s$,  it is assumed that the function $f$ to be approximated belongs to the same Hilbert space $\mathcal{H}_s$ and is approximated in $\Gamma^s$, justifying the use of the reproducing property.
\begin{theorem}\label{th2}
For observed data $f|_X = (f(x_1),f(x_2),...f(x_n))^T \in \R^n$ on $X \subset \Omega$, the approximation produced by the  $A_{s}f$ and its corresponding error $E^{s}$ satisfies 
\begin{equation}
\lim_{s \to S_a}\Big|\Big<(A_sf),E^s \Big>_{\mathcal{H}_s} \Big| \  \leq ||C_{S_a}||_{\infty} \sqrt{n}(TOL)
\end{equation}
where $C_{S_a} = \Big(|c_1|,|c_2|,...|c_{l_{S_a}}| \Big)^T \in \R^{l_{S_a}}$ contains the modulus of coordinates of projection at $S_a$, n is the number of observations and TOL is the $2$-$norm$ convergence error tolerance in the hierarchical algorithm. 
\end{theorem}
\begin{proof}
Beginning with the Inner products in the RKHS $\Big| \Big<(A_s f),E^s \Big>_{\mathcal{H}_s} \Big|$
\begin{align*}
\Big| \Big<(A_s f),E^s \Big>_{\mathcal{H}_s} \Big| &= \Big| \Big<\sum_{x_j \in X_s}K^s(\cdot,x_j)c_j, f - A_{s}f \Big >_{\mathcal{H}_s}\Big|\\
&= \Big| \Big<\sum_{x_j \in X_s}K^s(\cdot,x_j)c_j, f \Big>_{\mathcal{H}_s} - \Big<\sum_{x_j \in X_s}K^s(\cdot,x_j)c_j,A_{s}f \Big>_{\mathcal{H}_s}\Big| \\
&= \Big|\sum_{x_j \in X_s}f(x_j)c_j - \sum_{x_j \in X_s}(A_{s}f)(x_j)c_j\Big| \quad  \text{(Reproducing Property)} \\
&=\Big|\sum_{x_j \in X_s}c_j\Big(f(x_j)- (A_{s}f)(x_j)\Big)\Big| \\
& \leq\sum_{x_j \in X_s}\Big|c_j\Big(f(x_j)- (A_{s}f)(x_j)\Big)\Big|\quad  \text{(Triangle Inequality)} \\
& \leq ||C_{s}||_{\infty} \sum_{x_j \in X_s}\Big|\Big(f(x_j)- (A_{s}f)(x_j)\Big)\Big| \\
& \leq ||C_{s}||_{\infty} \sum_{x_j \in X}\Big| \Big(f(x_j)- (A_{s}f)(x_j)\Big)\Big|\quad (X_s \subseteq X)  \\
 &= ||C_{s}||_{\infty} ||E^s|_X||_1
\end{align*}

Here $||C_{s}||_{\infty}$ is the $\infty$ norm of $C_s = [|c_1|, |c_2|,...,|c_{l_s}|]^T$. 

\begin{align*}
|<(A_s f),E^s>_{\mathcal{H}_s}| &\leq ||C_{s}||_{\infty} ||E^s|_X||_1 \\
&\leq ||C_{s}||_{\infty} \sqrt{n}||E^s|_X||_2 \text{(Cauchy-Schwarz Inquality)}
\end{align*}

However, from the convergence criteria in the proposed Algorithm, we know $||E^s|_X||_2 \leq TOL$ as $s \to S_a$. Hence the theorem follows
\end{proof}
\section{Approximation Properties and Confidence Intervals}

This section provides estimates quantifying the quality of approximations generated by the proposed algorithm at each scale.   The first result is a direct application of theorem (\ref{th2}) 

\subsection{Approximation Properties}

\begin{corollary}
For observed data $f|_X = (f(x_1),f(x_2),...f(x_n))^T \in \R^n$ on $X \in \Omega$, the approximation $(A_{S_a}f)$ produced by the hierarchical algorithm at the convergence scale ($S_a$), follows the Pythagoras Theorem in the limit $TOL \to 0$. i.e,

\begin{equation}
|| f||^2_{\mathcal{H}_{S_a}}= || A_{S_a} f ||^2_{\mathcal{H}_{S_a}} + || f - A_{S_a}f ||^2_{\mathcal{H}_{S_a}}
\end{equation}
\end{corollary}

\begin{proof}

Let $f$ be the function to be approximated. As stated earlier, it was discretely observed at $X \in \Omega$ leading to the restricted function $f|_X$. Starting with the norm of the function in the Hilbert Space

\begin{align*}
||f||^2_{\mathcal{H}_s} & =  ||(A_s f)||^2_{\mathcal{H}_s} + ||f - (A_s f)||^2_{\mathcal{H}_s} + 2|<(A_s f),f - (A_s f)>_{\mathcal{H}_s}| \\
&= ||(A_s f)||^2_{\mathcal{H}_s} + ||f - (A_s f)||^2_{\mathcal{H} _s}+ 2|<(A_s f),E^s>_{\mathcal{H}_s}|
\end{align*}

Using theorem  \ref{th2}, at $s = S_a$ in the limit $TOL \to 0$, the result follows
\end{proof}

Now, we will provide results related to uniqueness and quality of solution. 

\begin{theorem}
For observed data $f|_X = (f(x_1),f(x_2),...f(x_n))^T \in \R^n$ on $X \in \Omega$, the approximation $(A_{s}f) \in \Gamma^s$  in the limit $TOL \to 0$ produced by the proposed algorithm at the convergence scale $s = S_a$ is
\begin{enumerate}
\item the unique orthogonal projection to $f$
\item the unique best approximation to $f$ with respect to $||\cdot|| _{\mathcal{H}_s}$
\end{enumerate}
\end{theorem}

\begin{proof}
(1): $A_sf$ is a unique orthogonal projection of $f$ on $\Gamma^s$ if $(f-A_sf) \perp \Gamma^s$. Again using the reproducing property of RKHS
\begin{align*}
\Big| \Big<\sum_{x_j \in X_{s}}c_jK^s(\cdot,x_j),f - A_sf \Big>_{\mathcal{H}_s}\Big| 
&= \Big| \Big< \sum_{x_j \in X_s}c_jK^s(\cdot,x_j),f \Big>_{\mathcal{H}_s} - \Big<\sum_{x_j \in X_s}c_jK^s(\cdot,x_j) ,A_sf \Big>_{\mathcal{H}_s}\Big| \\
&=\Big|\sum_{x_j \in X_s} c_j f(x_j) - \sum_{x_j \in X_s} c_j (A_sf)(x_j)\Big|  
\end{align*}
The rest of the proof is similar to steps in the proof of theorem \ref{th2} showing the considered native Hilbert space inner product vanishes as $TOL \to 0$

(2): Considering any $d_s \in \Gamma^s$. Then at $s = S_a$, $<d_s - A_s f, f - A_sf>_{\mathcal{H}_s} \to 0 $ as $TOL \to 0$  (by part 1)
Therefore, 
\[||d - f||^2_{\mathcal{H}_s} = ||d - A_sf + A_sf - f||^2_{\mathcal{H}_s} = ||d - A_sf||^2_{\mathcal{H}_s} + ||A_sf - f||^2_{\mathcal{H}_s}
\]
%
%
\[ \Rightarrow ||A_sf - f||^2_{\mathcal{H}_s} < ||d - f||^2_{\mathcal{H}_s} \quad d_s \neq A_sf\]
which establishes the optimality of the approximation
\end{proof}
Now, we will analyze the error functional and pointwise error bounds. At scale $s$, we are searching for a solution in the space $\Gamma^s$. Our approximation is of the form
\[
(A_sf)(x) = \sum_{x_j \in X_s} K^s(x,x_j)c_j
\]
Any $A_sf \in \Gamma^s$ can also be expressed as
\[
(A_sf)(x) = \lambda^y K^s(x,y); \quad \text{ where }\lambda = \sum_{x_j \in X_s} c_j \delta_{x_j}
\]
$\delta_x$ is the evaluational functional for $f$, i.e.  $\delta_x(f) = f(x)$ and $\lambda^y$ is application of linear functional $\lambda$ on $y$. Thus, we have the dual space
%
$\mathcal{H}^{*}_s = \Bigg\{ \sum_{x_j \in X_s} c_j \delta_{x_j} \Bigg\}$.
 
Next, we present a result which provides an inner product representation for approximation at a point $x \in \Omega$. The objective here is to show the optimality of the approximation $A_sf$ by requiring vanishing gradient for the norm of the Error Functional in $\mathcal{H}_s$.
\begin{theorem}\label{errfunc}
If  the error functional at scale $s$ is  represented as $\varepsilon^s_x = \delta^s_x - M_s(x)^T \delta^s_X$ $\ni$ 
\begin{equation}
|\varepsilon^s_x(f)|= |f(x) - (A_sf)(x)| = | \delta^s_x(f) - M_s(x)^T \delta^s_X(f)|
\end{equation}
Then the optimal value of $M_s(x) = [M^1_s(x),M^2_s(x),M^3_s(x)....,M^n_s(x)]^T \in \R^n $ which minimizes the error functional norm $||\varepsilon^s_x||^2_{{\mathcal{H}}_s}$ satisfies the inner product
\begin{equation}\label{inner}
 <f|_X, \hat{M}_s(x)> = (A_sf)(x)
\end{equation} 
 establishing the optimality of approximation $(A_sf)$ at each scale s.
\end{theorem}

\begin{proof}

We begin by expressing the error functional norm
\begin{align*}
||\varepsilon^s_x||^2_{\mathcal{H}_s}  &= < \delta^s_x - M_s(x)^T \delta^s_X,  \delta^s_x - M_s(x)^T \delta^s_X>_{\mathcal{H}_s}\\
&= ||\delta^s_x||^2_{{\mathcal{H}}_s} - 2M_s(x)^T \delta^s_X\delta^s_x + M_s(x)^T \delta^s_X{\delta^s}^T_XM_s(x)
\end{align*}
Now, based on the property of dual space, we know at scale s, 
\[
<\delta^s_a,\delta^s_b>_{\mathcal{H}^{*}_s}  = K^s(a,b) 
\]
Also, let
\[R_s(x) = \delta^s_X \delta^s_x = (K^s(x,x_1), K^s(x,x_2),....,K^s(x,x_n))^T \in \R^{n}\]
\[G_s = \delta^s_X {\delta^s}^T_X  \quad \text{(Gaussian kernel at scale s)}\]
Therefore, we have
\begin{align}\label{err}
||\varepsilon^s_x||^2_{\mathcal{H}_s}  &= ||\delta_x||^2_{{\mathcal{H}}_s} - 2M_s(x)^T R_s(x) + M_s(x)^T G_s M_s(x)
\end{align}
Differentiation yields the optimal $\hat{M}_s(x)$ that minimizes $||\varepsilon^s_x||^2_{\mathcal{H}_s}$, as the the solution of 
\begin{equation}\label{solve}
G_s M_s(x) = R_s(x)
\end{equation}
From Algorithm \ref{multiscale_algo}, we know $G_s$ need not be full rank. Using matrix A (as in Algorithm \ref{multiscale_algo}), constructing $W ( = AG_s)$ and a Column pivoted QR decomposition 
$WP_s = QR$,
%
and, applying the permutation operator $P_s$ on equation (\ref{solve}) yields
\begin{equation}\label{solve2}
P_s R_s(x) = P_s G_s M_s(x)
\end{equation}
Since $G_s$ is a symmetric operator, sampling the first $l_s$ rows of $P_s G_s$, where $l_s$ represents the numerical rank of $G_s$, will produce ${B^s}^T$, the transpose of the bases considered at scale $s$.
Correspondingly, sampling the respective values in $R_s(x)$ produces the restriction of $R_s(x)$ to set $X_s$ (represented as $R_s(x)|_{X_s}$). Therefore, restricting system (\ref{solve2}) to equations only corresponding to $X_s$
\begin{equation}
R_s(x)|_{X_s} = {B^s}^T M_s(x)
\end{equation}
where we use the Moore-Penrose inverse for getting the optimal projection of $R_s(x)|_{X_s}$. Thus, optimal solution for $M_s(x)$ is given by
\begin{align}
\hat{M}_s(x) &= {B^{s}} ({B^{s}}^T B^{s})^{-1} R_s(x)|_{X_s}
={{B^s}^{\dagger} }^TR_s(x)|_{X_s}\label{opt}
\end{align}
Now computing the inner product $<f|_X, \hat{M}_s(x)>$
\begin{align*}
<f|_X, \hat{M}_s(x)> &= <f|_X,{{B^s}^{\dagger} }^TR_s(x)|_{X_s} > 
= <{B^s}^{\dagger}f|_X,R_s(x)|_{X_s} >\\
&= \sum_{x_j \in X_s}c_j K(x,x_j)
= (A_sf)(x)
\end{align*}
\end{proof}
Now, we  provide approximation results for optimal point evaluational functional. For the approximation at each scale, we can  obtain the variance associated with that prediction by minimizing the squared error in evaluation at $x$ in the native reconstruction space. However, since there is no noise or uncertainty associated with the data, this variance will  be zero.
\begin{theorem}
For observed data $f|_X = (f(x_1),f(x_2),...f(x_n))^T \in \R^n$ on $X \subset \Omega$, there exists at each scale, a point approximation functional ($\psi^s_x$) for the optimal approximation $A_sf$ such that $\psi^s_x(f) = (A_sf)(x)$. This functional has a representation
\begin{equation}\label{delsx}
\psi^s_x = R^T_s(x)|_{X_s} {B^{s}}^{\dagger} \delta_X
\end{equation}
Also, it produces the minimized error functional in the sense
\begin{equation}\label{varr}
\min ||\varepsilon^s_x||^2_{\mathcal{H}_s}  = \min_{\gamma^s_x} ||\delta^s_x - \gamma^s_x||^2_{{\mathcal{H}}_s} \quad \text{       with optimal: }{\gamma^s_x} = \psi^s_x
\end{equation}
\end{theorem} 

\begin{proof}
We begin with the expression for the norm of the error functional in equation (\ref{err})
\[
||\varepsilon^s_x||^2_{\mathcal{H}_s} = ||\delta^s_x||^2_{{\mathcal{H}}_s} - 2M_s(x)^T R_s(x) + M_s(x)^T G_s M_s(x)
\]

since, it is minimized when $M_s(x)$ is given by equation (\ref{opt}). Now, 
writing the optimal approximation functional at x as:
\[ \psi^s_x(f) =  A_sf(x) = M^T_s(x) \delta_X (f) = R_s^T(x)|_{X_s} {B^s}^{\dagger} \delta_{X}(f) \]

which completes the proof for equation (\ref{delsx}). Now, considering the error functional at x and putting the optimal $M_s(x)$ in equation (\ref{err}), we get
\[
\min ||\varepsilon^s_x||^2_{\mathcal{H}_s} = ||\delta^s_x||^2_{{\mathcal{H}}_s} - 2 R_s^T(x)|_{X_s} {B^s}^{\dagger} R_s(x) + R_s^T(x)|_{X_s} {B^s}^{\dagger} G_s {{B^s}^{\dagger}}^T R_s(x)|_{X_s} 
\]

Recognizing $\psi^s_x$ from equation (\ref{delsx}) and substituting in the above equation.

\[
\min ||\varepsilon^s_x||^2_{\mathcal{H}_s} = ||\delta^s_x||^2_{{\mathcal{H}}_s} - 2 <\psi^s_x,\delta^s_x> + ||\psi^s_x||^2_{{\mathcal{H}}_s} = ||\delta^s_x - \psi^s_x||^2_{{\mathcal{H}}_s}
\]

Which completes the proof for equation (\ref{varr}).
\end{proof}

Before moving forward, here we also provide a result for bounding the Error functional.

\begin{corollary}
For the error functional defined in Theorem \ref{errfunc}, the absolute error at x is  bounded in the sense
\begin{equation}\label{ubound}
|\varepsilon^{s}_x(f)| \leq \sqrt{1 - M_s(x)^TR_s(x)} ||f||_{\mathcal{H}_s}
\end{equation}
\end{corollary}
\begin{proof}
Using the relation,
\[
|\varepsilon^{s}_x(f)| \leq ||\varepsilon^{s}_x||_{\mathcal{H}_s} ||f ||_{\mathcal{H}_s}
\]
again using equation (\ref{err})
\[
||\varepsilon^s_x||^2_{\mathcal{H}_s} = ||\delta^s_x||^2_{{\mathcal{H}}_s} - 2M_s(x)^T R_s(x) + M_s(x)^T G_s M_s(x)
\]
Now, since (\ref{err}) is minimized when we satisfy (\ref{solve}). Also since, $||\delta^s_x||^2_{{\mathcal{H}}_s} =1$.  Putting the values in (\ref{err}) we get 
\[
\min ||\varepsilon^s_x||^2_\mathcal{H} = 1 - M_s(x)^T R_s(x)
\]
Therefore result in equation (\ref{ubound}) follows.
\end{proof}

\subsection{Confidence and Prediction Intervals}
Confidence intervals in general are a measure of our belief in the estimated approximation. 
Prediction intervals on the other hand refer to the bounds  around the mean fit, where a future datapoint is expected to fall. This is crucial information in conjunction with the sparse representation, as even when we are not able to capture the function accurately at initial scales, we have an estimate of the expected behavior of the observations.
Algorithm \ref{multiscale_algo} makes predictions at each scale based on the corresponding sparse representation ($D^s_{sparse}$). Here if the error in approximation is greater than user defined tolerance, i.e.
%
$ ||f|_X -(A_sf)|_X||_2 \geq TOL  $
 , then it signifies that Approximation $A_sf$ needs additional degrees of freedom to capture the underlying data generation process (as in principle it is the best possible approximation, given a  basis set). 
 For modeling this error value, we consider a model  of form
\begin{equation}
f(x_i) = (A_sf)(x_i) +\epsilon^s_i \quad \text{ with } x_i \in X \text{ and } \epsilon^s_i \sim N(0,\sigma^2_s I)
\end{equation}
Therefore the sampling distribution for $f|_X$ would be given as
\[
p\Big(f|_X\Big|(A_sf)|_X,\epsilon^s \Big) \sim N(B^sC_s, \sigma^2_s I)
\]
Since, we know that the projection coordinate is given as $
C_s = ({B^s}^TB^s)^{-1}{B^s}^T (f|_X) $,
%
\begin{align*}
Cov(C_s) &= ({B^s}^TB^s)^{-1}{B^s}^T Cov(f|_X) B^s ({B^s}^TB^s)^{-1} = \sigma^2_s ({B^s}^TB^s)^{-1} \text{ and, }
\end{align*}
\begin{align*}
\mathbb{E}[C_s] =  ({B^s}^TB^s)^{-1}{B^s}^T \mathbb{E}[f|_X] = ({B^s}^TB^s)^{-1}{B^s}^TB^sC_s = C_s
\end{align*}
which shows an unbiased estimator. Hence $C_s$ is the best unbiased approximation at scale s given bases $B^s$. Now for computing the distribution of response $f$ at some $x_* \in \Omega$
\[
\mathbb{E}[A_sf(x_*)] = B^s_* \mathbb{E}[C_s] = A_sf(x_*) 
\]
\[
Var(A_sf(x_*)) = B^s_* Cov(C_s){B^s_*}^T = \sigma^2_s B^s_* ({B^s}^TB^s)^{-1} {B^s_*}^T
\]
For an estimated value of $\sigma$ ($\hat{\sigma}$), we can write the standard deviation of prediction at $x_*$ 
\[
\widehat{stdev}(A_sf(x_*)) = \hat{\sigma_s} \sqrt{ B^s_* ({B^s}^TB^s)^{-1} {B^s_*}^T}
\]
Therefore if we use t-distribution for confidence bounds, we get the following $100(1 - \alpha)\%$ confidence intervals for $E[A_sf(x_*)]$ 
\begin{equation}
(A_sf)(x_*) \pm \widehat{stdev}(A_sf(x_*))  \cdot t\Big( 1 - \frac{\alpha}{2}, n-l_s \Big)
\end{equation}
Here $l_s$ is again the numerical rank of $G_s$ and $n$ is the original number of observations made.
Now for prediction interval, we know, $
var(f(x_*) - (A_sf)(x_*)) = var(\epsilon^s) + var(A_sf(x_*))
$,
therefore,
\[
\widehat{stdev}(f(x_*) - (A_sf)(x_*)) = \hat{\sigma_s}\sqrt{1 + B^s_* ({B^s}^TB^s)^{-1} {B^s_*}^T}
\]
Hence the $100(1 - \alpha)\%$ prediction intervals are given as

\begin{equation}
(A_sf)(x_*) \pm \widehat{stdev}(f(x_*) - (A_sf)(x_*))  \cdot t\Big( 1 - \frac{\alpha}{2}, n-l_s \Big)
\end{equation}
For the estimated value of $\sigma^2$, we use its unbiased estimation at scale s given as
\begin{equation}\label{var}
\hat{\sigma}^2 = \frac{||f|_X - A_sf|_X||_2^2}{n-l_s}
\end{equation}

\section{Stability Properties}

In this section we provide bounds related to stability of approximations obtained by the proposed algorithm. The first result bounds the approximation at scale $s$ with respect to the $L_{\infty}$ topology for some compact domain $\Omega \in \R^d$.
\begin{theorem}\label{stability}
For observed data $f|_X = (f(x_1),f(x_2),...f(x_n))^T \in \R^n$ on $X \in \Omega$, the approximation   $A_{s}f$ at any scale s is bounded in the $L_{\infty}(\Omega)$ norm as 
\begin{equation}
||A_sf||_{L_{\infty}(\Omega) } \leq P^s_{\infty} ||f||_{L_{\infty}(\Omega) }  \text{, where }
\end{equation}
\begin{equation}
{(\sigma_{smax}({B^s}^T))}^{-1} \leq P^s_{\infty}\leq \sum_{j = 1}^n \sqrt{D_s(j,j)}
\end{equation}
where $\sigma_{smax}$ denotes the largest singular value and  $D_s$ is  obtained by applying the extension operator ${{B^s}^{\dagger}}^T$ on $G_{X_s} = R_s(x)|_{X_s} R_s(x)|^T_{X_s}$, i.e.
$D_s = {{B^s}^{\dagger}}^T G_{X_s} {B^s}^{\dagger}$

\end{theorem}

\begin{proof}
Expressing  $(A_sf)$ in terms of the inner product as in equation (\ref{inner}). 

\begin{align*}
||A_sf||_{L_{\infty}{\Omega}} &= \max_{x \in \Omega}|A_sf(x)| =  \max_{x \in \Omega}\Big| \sum_{x_j \in X} f(x_j) M^j_s(x) \Big|\\
& \leq \max_{x \in \Omega} \sum_{x_j \in X}| f(x_j) M^j_s(x)|\quad  \text{(Triangle Inquality)} \\
& \leq \max_{x \in \Omega} \sum_{x_j \in X}| f(x_j)|\cdot | M^j_s(x)|\\
& \leq P^{s}_{\infty} \cdot ||f||_{L_{\infty}(\Omega)} \quad \text{where} \quad P^{s}_{\infty} = \max_{x \in \Omega} \sum_{j = 1}^n | M^j_s(x)|
\end{align*}
%
Let  $x^{*} \in \Omega$ be the data point at which $\sum_{j = 1}^n | M^j_s(x)|$ is maximized. Therefore 
\begin{align*}
P^s_{\infty} &= \sum_{j = 1}^n |M^j_s(x^{*})| 
= \sum_{j = 1}^n |\delta^s_{x^{*}}M^j_s| \quad  \text{(Point Evaluational functional)}\\
&\leq \sum_{j = 1}^n ||\delta^s_{x^{*}} ||_{{\mathcal{H}}_s} \cdot ||M^j_s||_{{\mathcal{H}}_s} =  \sum_{j = 1}^n ||M^j_s||_{{\mathcal{H}}_s} \quad \text{Since} \  ||\delta^s_{x^{*}} ||_{{\mathcal{H}}_s} = 1
\end{align*}
Now, we know
\begin{align}
<M^j_s,M^j_s>_{{\mathcal{H}}_s} &= e^T_j {{B^s}^{\dagger}}^T R_s(x)|_{X_s} \cdot R_s(x)|^T_{X_s} {B^s}^{\dagger} e_j \\
&= e^T_j {{B^s}^{\dagger}}^T G_{X_s}   {B^s}^{\dagger} e_j 
= e^T_j D_s e_j 
= D_s(j,j)\label{mjmj}
\end{align}
Therefore
\begin{equation}\label{lefteq}
P^s_{\infty} \leq \sum_{j = 1}^n ||M^j_s||_{{\mathcal{H}}_s} =  \sum_{j = 1}^n \sqrt{D_s(j,j)}
\end{equation}
For lower bound on $P^s_{\infty}$,
\begin{align}
P^s_{\infty} &=  \max_{x \in \Omega} || M_s(x)||_1 \geq \max_{x \in \Omega} || M_s(x)||_2 \quad (\text{since}\  ||\cdot||_1 \geq ||\cdot||_2 )\label{p}
\end{align}
For a tall thin matrix B, ${{B^{\dagger}}^T}^{\dagger} = B^T$. Therefore, $M_s(x)$ is the least squares solution of 
%
${B^s}^T M_s(x) = R_s(x)|_{X_s}$, and,
%
\begin{equation} \label{upper}
||{B^s}^T||_2 ||M_s(x)||_2 \geq ||R_s(x)|_{X_s}||_2
\end{equation}
\[
||M_s(x)||_2  \geq  {(\sigma_{smax}({B^s}^T))}^{-1}||R_s(x)|_{X_s}||_2
\]
where $\sigma_{smax}$ is the maximum singular value. Putting in equation (\ref{p}).
\begin{align*}
P^s_{\infty} & \geq \max_{x \in \Omega} {(\sigma_{smax}({B^s}^T))}^{-1}||R_s(x)|_{X_s}||_2 \\
& = {(\sigma_{smax}({B^s}^T))}^{-1} \max_{x \in \Omega} ||R_s(x)|_{X_s}||_2\\
& \geq  {(\sigma_{smax}({B^s}^T))}^{-1}   \quad (\text{Because}\   \max_{x \in \Omega} ||R_s(x)|_{X_s}||_2 > 1)
\end{align*}
Combining equation (\ref{lefteq}) and above result, the bounds on $P^s_{\infty}$ follow
\end{proof}
We will conclude this section by providing a bound on the approximation at any scale s at some point $x \in \Omega$
\begin{theorem}
For observed data $f|_X = (f(x_1),f(x_2),...f(x_n))^T \in \R^n$ on $X \in \Omega$, the absolute value of the approximation produced by the hierarchical algorithm $(A_{s}f)$ at scale $s \leq S_a < S_c$ at any point $x \in \Omega$ is  bounded as
\begin{equation}\label{absolute}
|(A_sf)(x)| \leq \sum_{j = 1}^n \sqrt{D_s(j,j)} \cdot ||f||_{{\mathcal{H}}_s}
\end{equation}
Where $D_s$ is defined as in  Theorem (\ref{stability}). Also if the convergence happens at the critical scale ($S_c$), then at convergence, bound (\ref{absolute}) can be simplified as
\begin{equation}
|(A_{S_c}f)(x)| \leq n \cdot \sigma_{max}(G^{-1}_{S_c})\cdot ||f||_{{\mathcal{H}}_s}
\end{equation}
Where $\sigma_{max}$ is the maximum eigenvalue operator
\end{theorem}
 \begin{figure*}\label{graph1}
\centering
\includegraphics[width=6in]{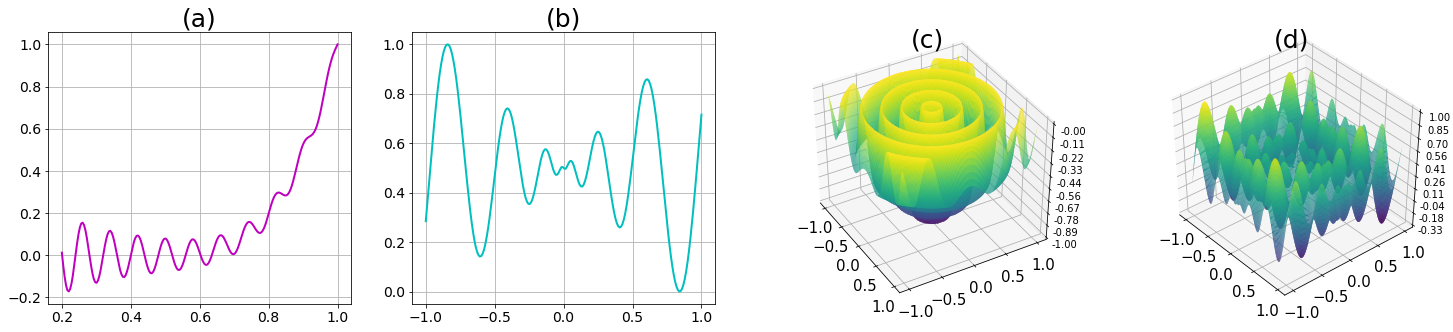}
\caption{Univariate (a: Gramancy and Lee function - Test function 1 (TF1); b: 1-D Schwefel function - TF2) and multivariante (c: Dropwave function - TF3; d: 2-D Schwefel function - TF4) test functions considered for studying the performance of Algorithm \ref{multiscale_algo}.}
\end{figure*}
\begin{proof}
Here again we begin with the absolute value of approximation produced by the proposed algorithm at scale s.
\begin{align*}
|(A_sf)(x)| & \leq \Big| \sum_{j = 1}^n f(x_j) M^j_s(x) \Big| \leq \sum_{i = 1}^n |f(x_j)||M^j_s(x)| = \sum_{j = 1}^n |\delta^s_{x_j}f||M^j_s(x)| \\
& \leq \sum_{j = 1}^n ||\delta^s_{x_j}||_{{\mathcal{H}}_s} ||f||_{{\mathcal{H}}_s} ||M^j_s||_{{\mathcal{H}}_s} = \sum_{j = 1}^n ||f||_{{\mathcal{H}}_s} ||M^j_s||_{{\mathcal{H}}_s} \quad \text{since: } (||\delta^s_{x_j}||_{{\mathcal{H}}_s} = 1)\\
& = \sum_{j = 1}^n \sqrt{D_s(j,j)} ||f||_{{\mathcal{H}}_s}
\end{align*}
The last step was carried out using equation (\ref{mjmj}). Now if the convergence happens at $s = S_c = S_a$, then $D_{S_c} = G_{S_c}^{-1}$ and therefore $D_{S_c}(j,j) = G_{S_c}^{-1}(j,j)$. Let $g(j,j) = G_{S_c}^{-1}(j,j)$
\begin{align*}
|(A_{S_c}f)(x)| & \leq \sum_{j = 1}^n \sqrt{g(j,j)} ||f||_{{\mathcal{H}}_s} \\
& \leq \sum_{j = 1}^n g(j,j) ||f||_{{\mathcal{H}}_s} \quad  \text{Since }g(j,j) \geq 1 \text{ for }1 \leq j \leq n\\
& \leq n \cdot \sigma_{max}(G^{-1}_{S_c}) ||f||_{{\mathcal{H}}_s}
\end{align*}
\end{proof}
\begin{remark}
We note that the bounds here are conservative and depend on the data set size $n$. The assumed global overlap of the basis functions leads to the loose upper bound. However, as Figure \ref{graph3} (in the next section) shows the basis functions have a rapid decay and attained bounds in practice are much smaller.
\end{remark}

\section{Results and Analysis}

This section analyzes the behavior of the proposed approach on a variety of datasets under different conditions. The first subsection here studies the performance on synthetic datasets. This is important, as here we know the truth and hence quantification of performance becomes feasible. The following subsections deal with application on real datasets. Here we take 2 different applications which test the performance of the proposed hierarchical algorithm.

 \begin{figure*}\label{graph2}
\centering
\includegraphics[width=5in]{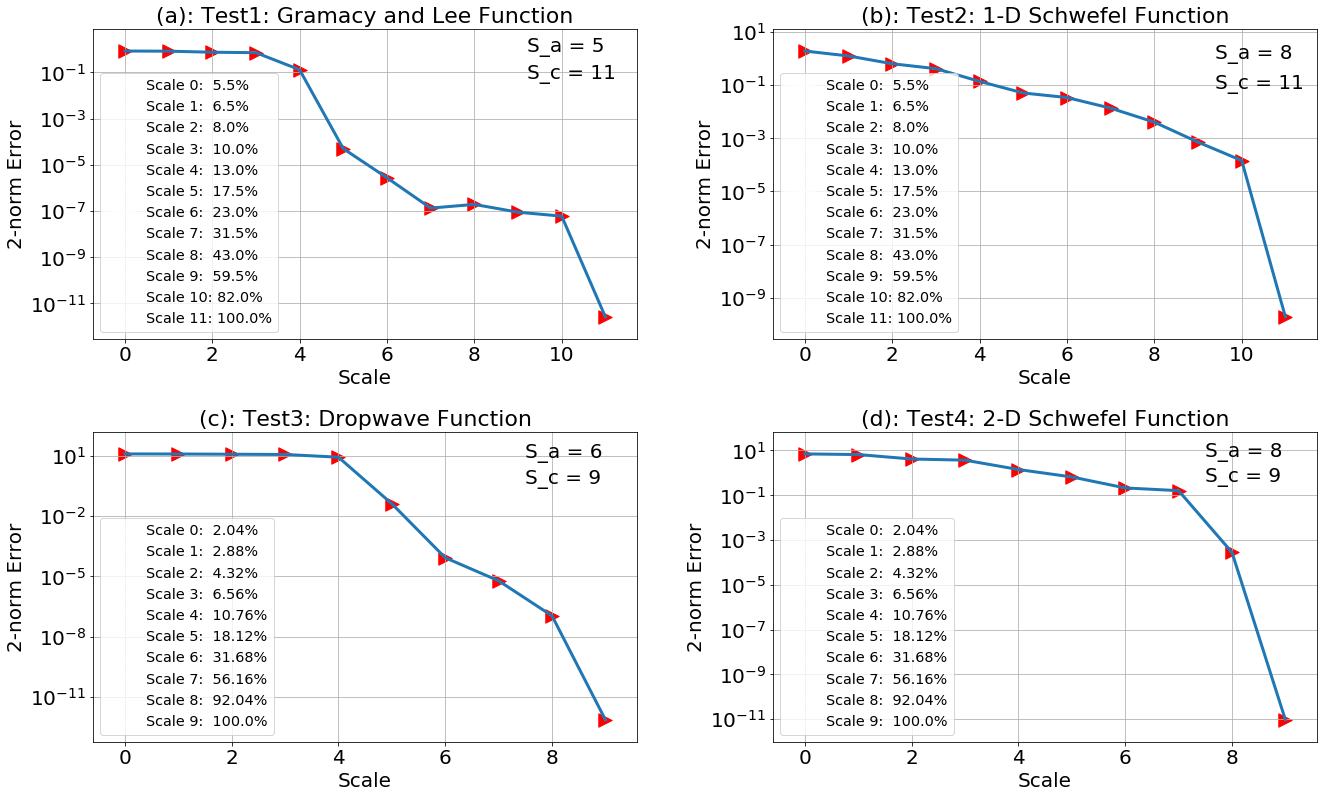}
\caption{Convergence behavior on the test functions measured in 2-norm prediction error on the observed data. Top row (a and b) shows the performance on univariate functions with bottom row for multivariate functions (c and d). 
Each of the plots  also show the Critical ($S_c$) and Convergence scale ($S_a$) along with the $\%$ of data sampled as the sparse representation  ($D^{s}_{sparse}$) at each scale.}
\end{figure*}

\subsection{Analysis on Synthetic Datasets}

Here we have chosen a set of 4 test functions (Figure \ref{graph1}) from literature \cite{simulationlib} providing our proposed algorithm, the sampled data to learn the underlying function. These test functions have been shown in Figure \ref{graph1}, and are mathematically expressed as:

\begin{itemize}
\item Test Function 1 (TF1): Gramacy and Lee Test function
\begin{equation}
f(x) = \frac{sin(10 \pi x)}{2x} + (x-1)^4 \quad \text{where } x \in [0.5, 2.5]
\end{equation}
\item Test Function 2 (TF2): 1-D Schwefel Function
\begin{equation}
f(x) = 418.9829 - x \cdot sin(\sqrt{|x|}) \quad \text{where } x \in [-500, 500]
\end{equation}
\item Test Function 3 (TF3): Dropwave Function
\begin{equation}
f(x,y) =-\frac{1+ cos(12 \sqrt{x^2 + y^2})}{0.5(x^2+y^2) + 2}\quad \text{where } x,y \in [-2, 2]
\end{equation}

\item Test Function 4 (TF4): 2-D Schwefel Function
\begin{equation}
f(x,y) = 837.9658 -  x \cdot sin(\sqrt{|x|}) - y \cdot sin(\sqrt{|y|})\quad \text{where } x,y \in [-500, 500]
\end{equation}

\end{itemize}
\begin{figure*}\label{graph5}
\centering
\includegraphics[width=6in]{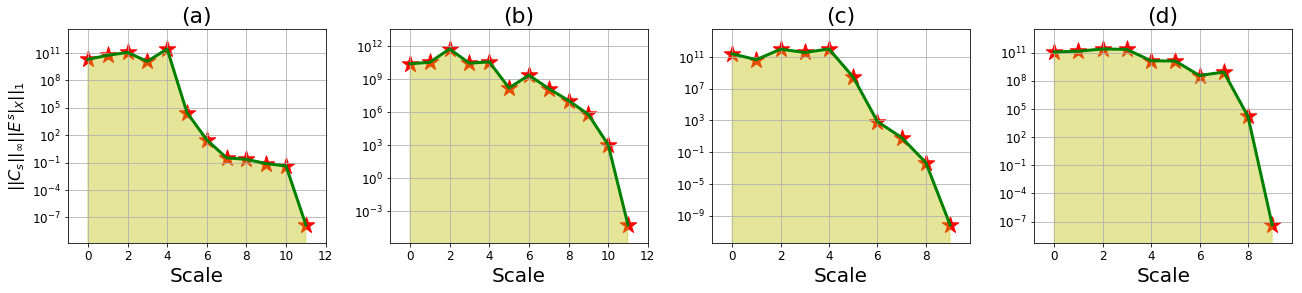}
\caption{Convergence measured as the decay of upper bound (equation (\ref{reseq})) to the inner product in the native Hilbert space between  approximation $A_sf$ and the approximation error $E^s$ for considered univariate (a: TF1 and b: TF2) and multivariate (c: TF3 and d: TF4) test functions. Here all 4 plots share common Y-axis label.}
\end{figure*}
 
After sampling the data from these test functions, for all axis (X and Y for TF1 and TF2, and X, Y and Z for TF3 and TF4 respectively) the values are normalized between -1 to 1 by dividing the measurements with the corresponding absolute maximum value along each axis. It should be noted here that for most of the analysis presented here, we have sampled 200 equidistant points for TF1 and TF2. And for TF3 and TF4, points are sampled on a $50 \times 50$ grid. The motivation here is to sample data points from the test functions and reconstruct the functions from the sampled data using the generated sparse representation ($D^s_{sparse}$). These functions were specifically chosen as they have a lot of curvature changes and multiple local minima and maxima, which makes learning the function form difficult. 

We begin with the analysis of the convergence of the algorithm on the test functions.
 \begin{figure*}\label{graph3}
\centering
\includegraphics[width=6in]{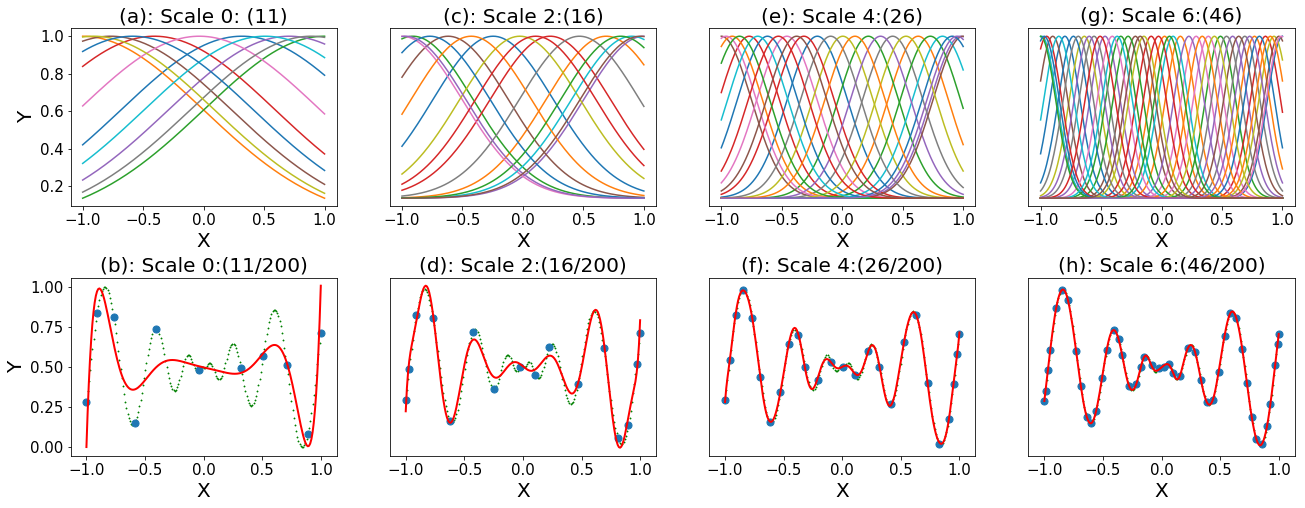}
\caption{Scalewise performance of Algorithm \ref{multiscale_algo} on Test function 2. The plots in the top row (a, c, e and g) show the density of basis functions at each scale, selected while identifying the sparse representation. The numbers in round bracket shows the cardinality of basis function set $B^s$. The bottom row (b, d, f and h) shows the corresponding scalewise reconstruction of the underlying function. The green curve  is the true function with blue points being the sparse representation and red curve is the reconstruction from the sparse representation}
\end{figure*}

\subsubsection{Convergence Behavior}
Figure \ref{graph2} shows the convergence behavior of the Algorithm \ref{multiscale_algo} by studying the $2$-$norm$ error of the prediction with respect to the original observations.  $TOL = 10^{-2}$ was used in Algorithm \ref{multiscale_algo} for generating these results. Following the notations used earlier, $S_a$ here represents the convergence scale with $S_c$  being the critical scale.  Figure \ref{graph2} also illustrates the proportion of dataset used at each scale $s$ for generating the approximation $A_sf$. It is essentially the proportion of the dataset used as the sparse representation.  Therefore, from figure \ref{graph2}  we can make inferences like, for test function 1, at scale 6 with 23\% of datapoints, the proposed algorithm was able to generate an approximation $A_6f$ which had a $2$-$norm$ error of less than $10^{-5}$. It should be noted here, that based on the  structure, the error measure reduces in a unique manner for all 4 test functions which have  different convergence scales. 

 \begin{figure*}\label{graph4}
\centering
\includegraphics[width=6in]{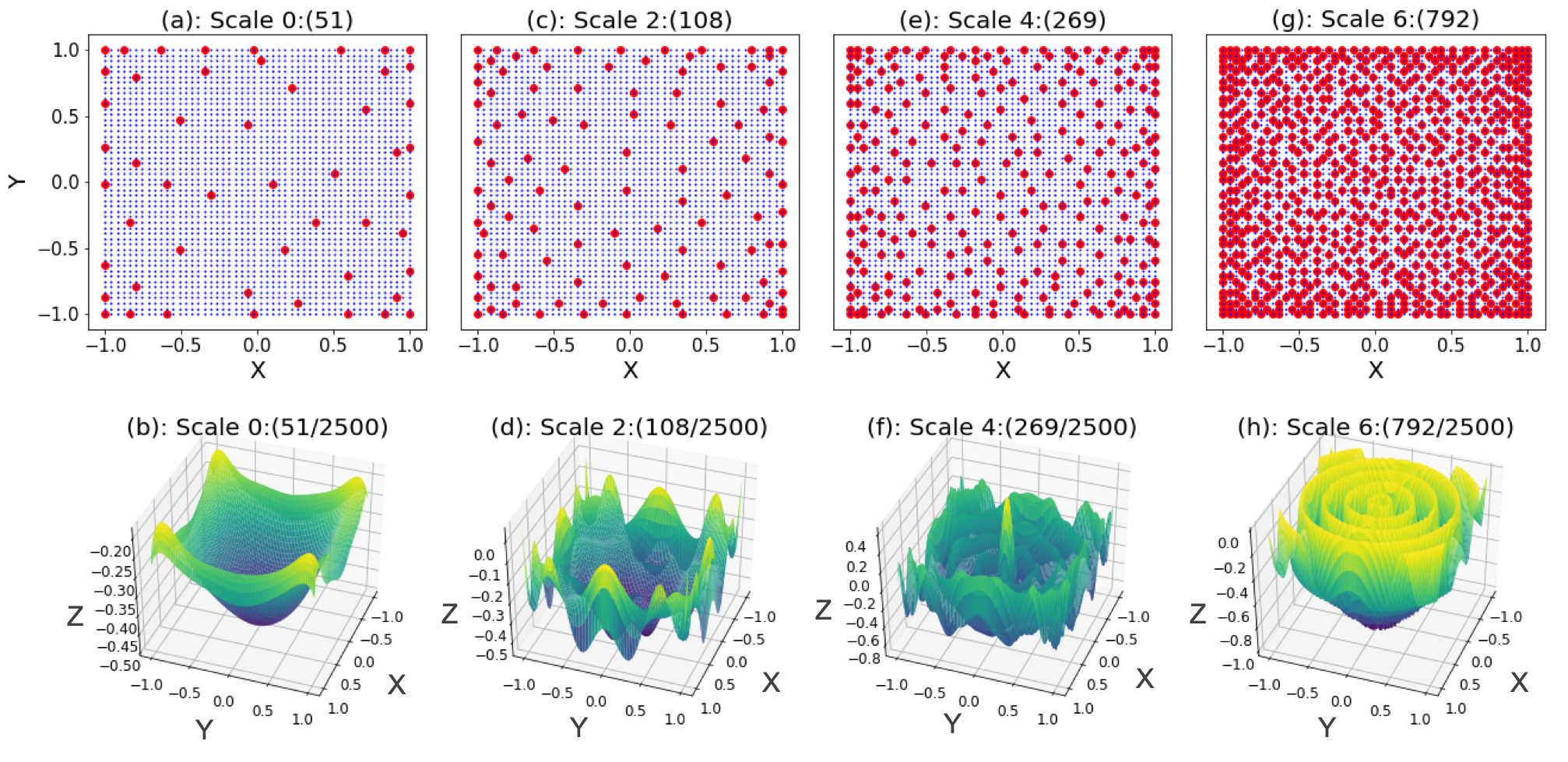}
\caption{Scalewise performance of Algorithm \ref{multiscale_algo} on Test function 3. Top row (a, c, e and g) shows the distribution of the sparse representation selected at multiple scales. Here we have shown the projection of the sparse representation on the X-Y plain for ease of presentation. The bottom row (b, d, f and h) shows the corresponding reconstruction for the dropwave test function from the respective $D^s_{sparse}$.}
\end{figure*}

Figure \ref{graph5} shows the convergence bounds from Theorem \ref{th2}. Here we have shown the results in $1$-$norm$, which state that at any scale s, in the Hilbert space
\begin{equation}
    \Big| \Big< (A_sf),E^s \Big>_{\mathcal{H}_s} \Big| \leq ||C_s||_{\infty} ||E^s|_X||_1
    \label{reseq}
\end{equation}
Figure \ref{graph5} plots the quantity on right in the above equation with increasing scales. The sharp drop in the this bound for all 4 test functions justifies the capacity of the algorithm to produce good approximations. Precisely, at higher scale in the native Hilbert space, the error in prediction approaches orthogonality with respect to the approximation $A_sf$.

\subsubsection{Scale dependent $D^s_{sparse}$ and corresponding reconstructions}

 \begin{figure*}\label{graph7}
\centering
\includegraphics[width=6in]{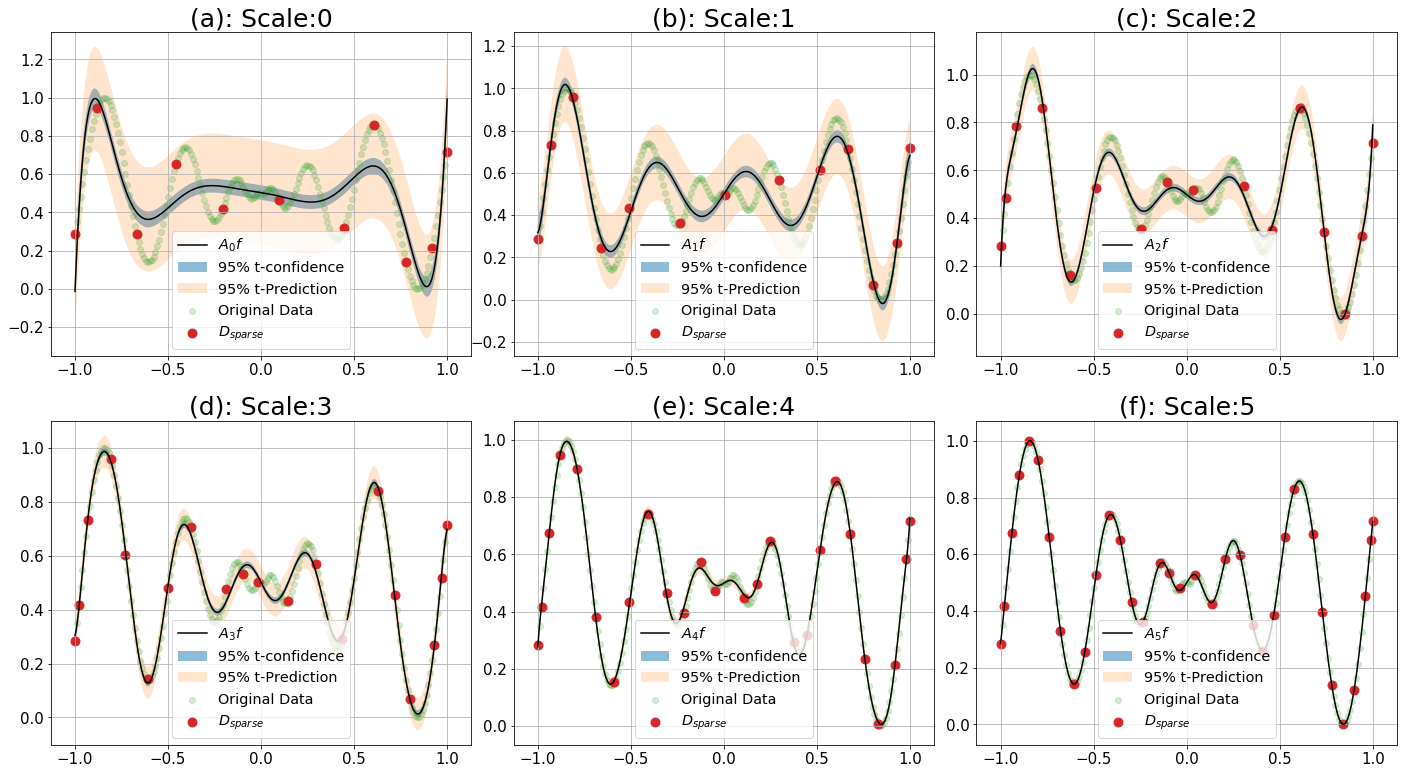}
\caption{Confidence and Prediction Intervals for reconstruction from $D^s_{sparse}$  from scale 0 to scale 5 for TF2. Along with the approximation produced at these scales (black curve), the  plots also show the sparse representation (red data points) selected with the $95\%$ t-confidence interval (thinner shaded region) and $95\%$t-prediction interval (broader shaded region).}
\end{figure*}

Figure \ref{graph3} and Figure \ref{graph4} show the behavior of the selected sparse representation and the approximation $A_sf$ produced with increasing scales. In this subsection we have only presented results for TF2 and TF3 for analyzing the performance of the algorithm in univariate and multivariate setting respectively. Starting with figure \ref{graph3}, the plots on the top row  show that with increasing scale, the support of the basis functions becomes narrower. This corresponds to  the increasing numerical rank of the kernel matrix $G_s$ with increasing $s$. The wider support of the basis functions in the initial scales also explains the corresponding over-smoothed approximations. This scenario is similar to behavior of approximation strategies with global bases (for example polynomial based approximation). In the bottom row, for every scale we have mentioned the number of points chosen as sparse representation (out of 200 points).  It should be noted here that the blue points represent the $D^s_{sparse}$ sampled from the smaller green data points. Here, our motivation is not just to show that with few points, the algorithm is able to learn the underlying function. But also, that the algorithm has an inbuilt capacity to choose a small set of representative points which can appropriately capture the function structure.

Figure \ref{graph4} shows the corresponding result for the 2-D wave functions. Since proper visualization of the basis functions for a surface is a little challenging, so here we have just shown the location of the points chosen in the sparse representation (in X-Y plane). One additional thing to be noted here is the higher density of sampling near the edges of the domain. This directly corresponds with the fact that at the edges, for matching the curvature appropriately, it needs more points as there is no scope of learning beyond the edges. The corresponding reconstruction also shows how the specific features are learned over the scales. Again, since points were sampled on a $50 \times 50$ grid. So $D^s_{sparse}$ consists of data points sampled from 2500 design points.

\subsubsection{Confidence and Prediction Intervals}

The results for this section have been shown in figure \ref{graph7}. The analysis is shown for scale 0 to scale 5. The green points are the original data and the red points show the sampled ones for $D^s_{sparse}$. The thinner (bluish) bands show the 95\% confidence interval on the estimated approximation at each scale. It should be noted here that, we have not carried out a full Bayesian analysis here. Instead, we have just used the fitting variance (equation \ref{var}) as a proxy for the variance and used it for scaling our interval. Specifically here we are using t-confidence bands which are suitable for smaller datasets (as compared to Gaussian bounds) and tend to be normal in the limit of larger datasets. The thicker band show the prediction interval. The main idea to be conveyed here is that if at a scale s, we have our sparse representation $D^s_{sparse}$, then along with these prediction intervals, we can make estimations of where all the deleted points and data points to be sampled in future would lie. One other way to say the same thing is, that if we are at a particular scale $s$, and if we fix the location on $x \in X$, then we can be 95\% confident that the mean of $y$ values observed at that particular x will lie within the blue bounds. Similarly for a fixed $x$, the broader salmon color bands show the range in which any new observation to be made in the future will lie.

 \begin{figure*}\label{graph7_5}
\centering
\includegraphics[width=3.1in]{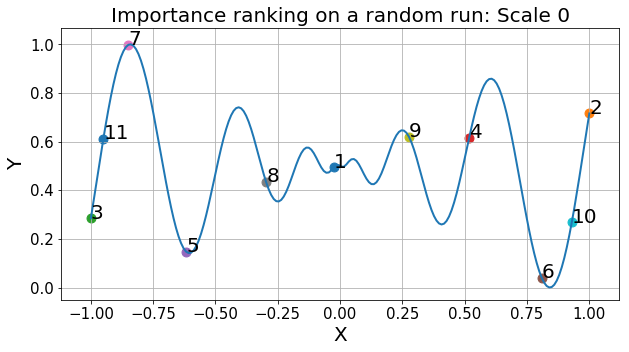}
\caption{Importance ranking of the sparse representation selected at scale 0 for TF2}

\end{figure*}

The results here also confirms the fact that with increasing scales both the bounds become very thin showing confidence in the approximation produced.

 \begin{figure*}\label{graph8}
\centering
\includegraphics[width=6in]{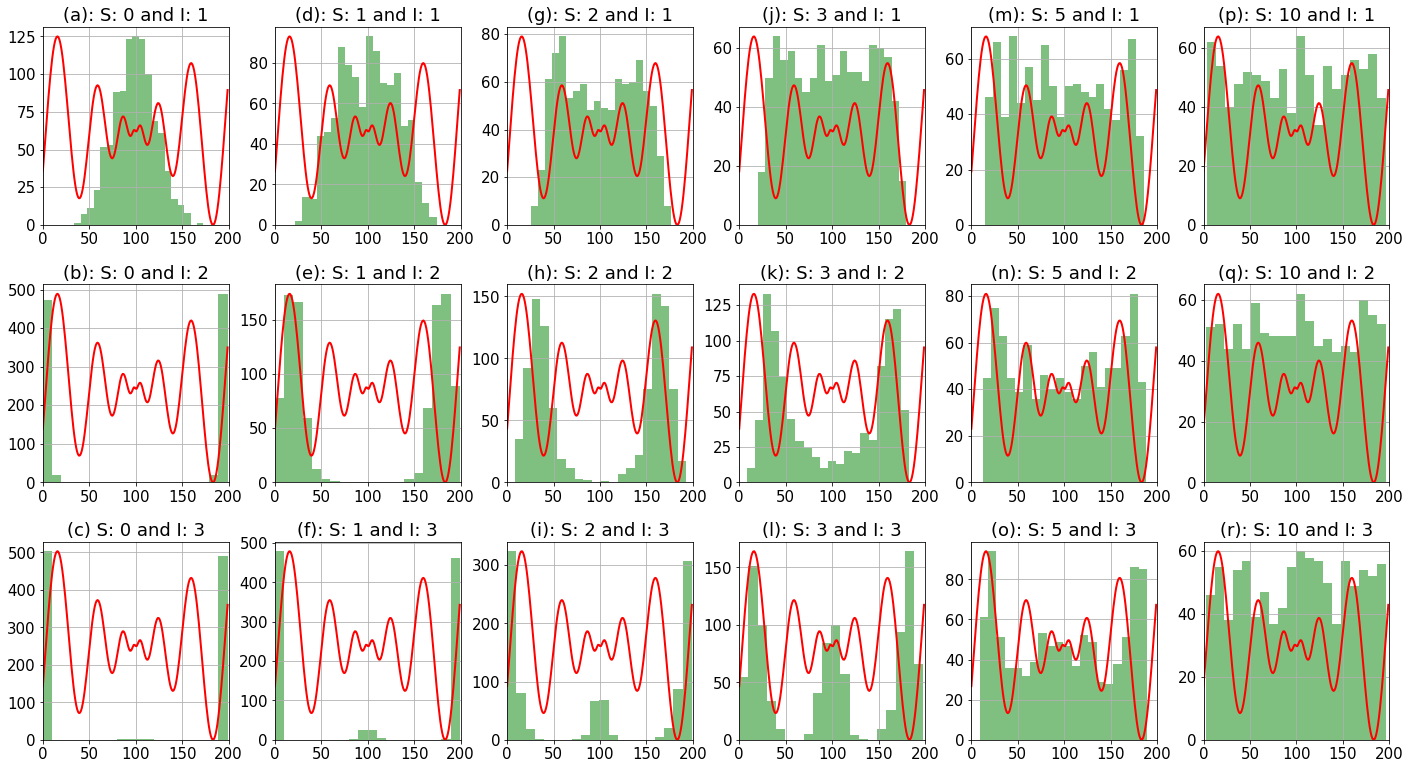}
\caption{Histogram of the top 3 most important points selected for TF2 (shown in red in each subplot). Here along the columns we have the increment in scale (S) and along the rows we have shown the histogram of the first, second and third most important point respectively (represented with abbreviation I for Importance).}
\end{figure*}

\subsubsection{Importance metric for design points}

This section aims at further exploring the application of the proposed algorithm. The results of the current study are presented in figure \ref{graph7_5} and \ref{graph8}. The idea is based upon the requirement that besides just getting the sparse representation at each scale, sometimes we also need to define an ordering of data points in $D^s_{sparse}$ in decreasing order of importance with respect to efficient reconstruction of $f$. This is important because if a measurement at a design point is of very high priority, then more resources could be engaged to measure that particular observation accurately. Also if we need to further compress the data, then in what sequence the datapoints can be removed. A test case is shown in figure \ref{graph7_5}. It contains the importance ranking computed based on the order in which the points were sampled while getting the sparse representation. It shows that the most important point (rank 1) at scale 0 is near the center of the function (and hence should rank last in the ordering for deletion). This confirms our belief that an observation near the center is crucial to capture the behavior of the function to be approximated. The second and third most important points (rank 2 and 3 respectively) are found to be at the very end of the curve which again is logical based on the fact that algorithm needs precise information to capture the function at the edges. These results also make sense because of the inherent symmetry of the 1-D schwefel function under consideration. It should be noted that the location of these important points depends on the nature of the function under consideration, so for any other functions, the location of important points might be very different as compared to the ones obtained in figure \ref{graph7_5}.

In order to get a more detailed picture of importance metric, we ran 1000 simulations of Algorithm \ref{multiscale_algo} and presented the distribution (for location) of the top 3 most important points. The results are presented in figure \ref{graph8}. Here the X-axis varies from 0 to 200 because we considered 200 sampled points from TF2 as our learning set. The analysis was run for scales 0, 1, 2, 3, 5, 10 for studying the behavior of the distribution. It could be very well seen here that all the weight of the distribution for the most important point is concentrated at the center of the domain for the initial scales (top row). For the same scales, the second (middle row) and third (bottom row) most important points have all their mass concentrated at the edges. However, if we move towards the critical scale (where all points are included in the sparse representation), all the points have approximately uniform importance density distribution. This is also expected as when all points are sampled, no single point is more important than the other.

\subsubsection{For non-uniform sampling of data}

We illustrate the performance of the proposed algorithm on non-uniformly sampled data (and fewer data points) in Figure \ref{graph9}. Here in plot (a), we can see the performance when the learning dataset composed of functional values at 40 randomly chosen locations. Here we have also shown the true function for visualizing the quality of the approximation generated. If we look closely at the reconstruction in figure \ref{graph9}(a) around $x = 1.4$, then the ability of the proposed approach to capture the respective peak in the underlying function even though there was not enough data to reflect it is clearly visible. Besides 40, the reconstruction has also been shown (in figure \ref{graph9}) for datasets consisting of 50, 60 and 70 points (b, c and d respectively) with promising results.

 \begin{figure*}\label{graph9}
\centering
\includegraphics[width=5in]{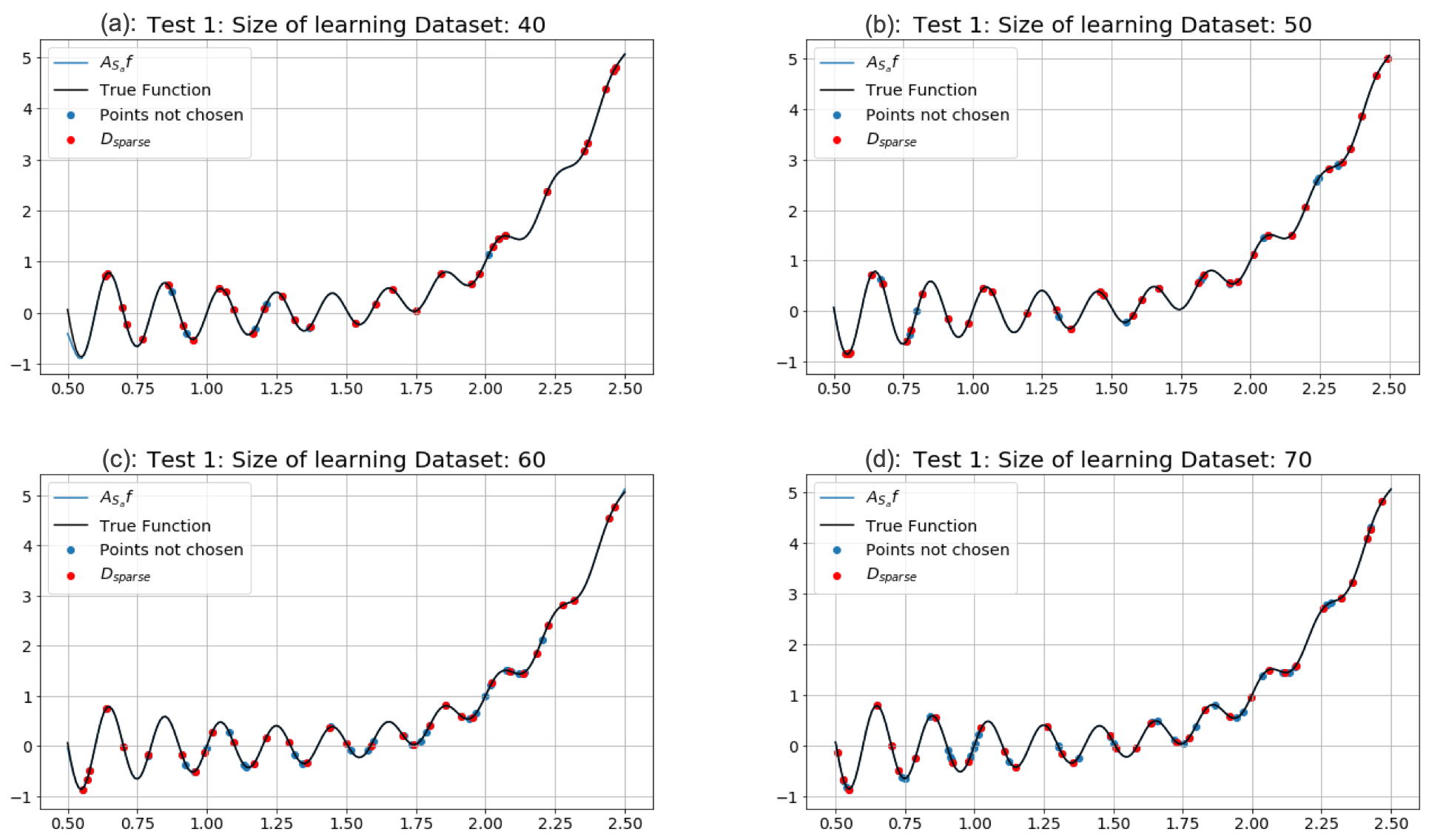}
\caption{Performance of Algorithm \ref{multiscale_algo} with non-uniformly sampled data on TF1. The sparse representation ($D^s_{sparse}$), the deleted data, the true function (black) and the approximation (blue) are all shown for proper comparison}
\end{figure*}

 \begin{figure*}\label{graph12}
\centering
\includegraphics[width=6in]{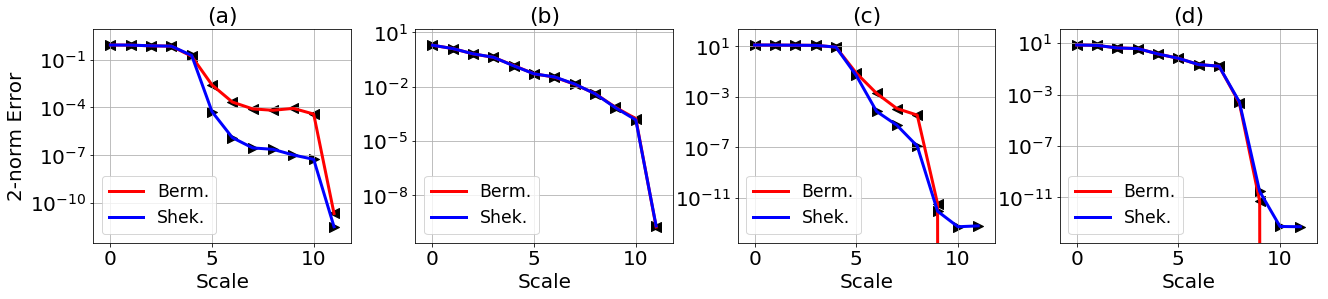}
\caption{Comparative behavior of the decay of $2$-$norm$ error for the Multiscale extension algorithm from \cite{bermanis2013multiscale} (represented as `Berm') and our hierarchical algorithm \ref{multiscale_algo} (represented as `Shek') }
\end{figure*}

Before moving forward with application on real datasets, the next subsection provides a comparison of the performance of Algorithm \ref{multiscale_algo} with algorithm 4 in \cite{bermanis2013multiscale}.

\subsubsection{Comparison with Algorithm 4 in Bermanis et. al}

As mentioned earlier,  \cite{bermanis2013multiscale} was one of the major motivators for the current work. In that paper  algorithm 4 for multiscale data sampling and function extension comes very close in behavior to our hierarchical approach. Briefly the idea there can be summarized as follows. Suppose the approximation at scale $s$ is represented as $H_sf$. Therefore starting with scale 0, $H_0f$ is the approximation to $f$ produced at scale $0$. Thus we can write
\[
(H_0f)|_X   \approx f|_X
\]

However, when we move further, the error orthogonal to the search space at the previous scale becomes the target function for the next scale. Therefore,
\[
(H_1f)|_X  \approx (f-(H_0f))|_X
\]
\[
\vdots
\]
\[
(H_sf)|_X  \approx (f-  \sum_{i = 0}^{s-1}H_if)|_X
\]

With a user defined error tolerance (err), the authors define the convergence scale ($s^*$) as the scale satisfying
\[
||(f-  \sum_{i = 0}^{s^*}H_if)|_X|| \leq err
\]

Therefore, final approximation to f is of the form
\begin{equation}\label{berm}
f \approx H_0f + H_1f + ... + H_{s^{*}}f
\end{equation}

 \begin{figure*}\label{graph13}
\centering
\includegraphics[width=6in]{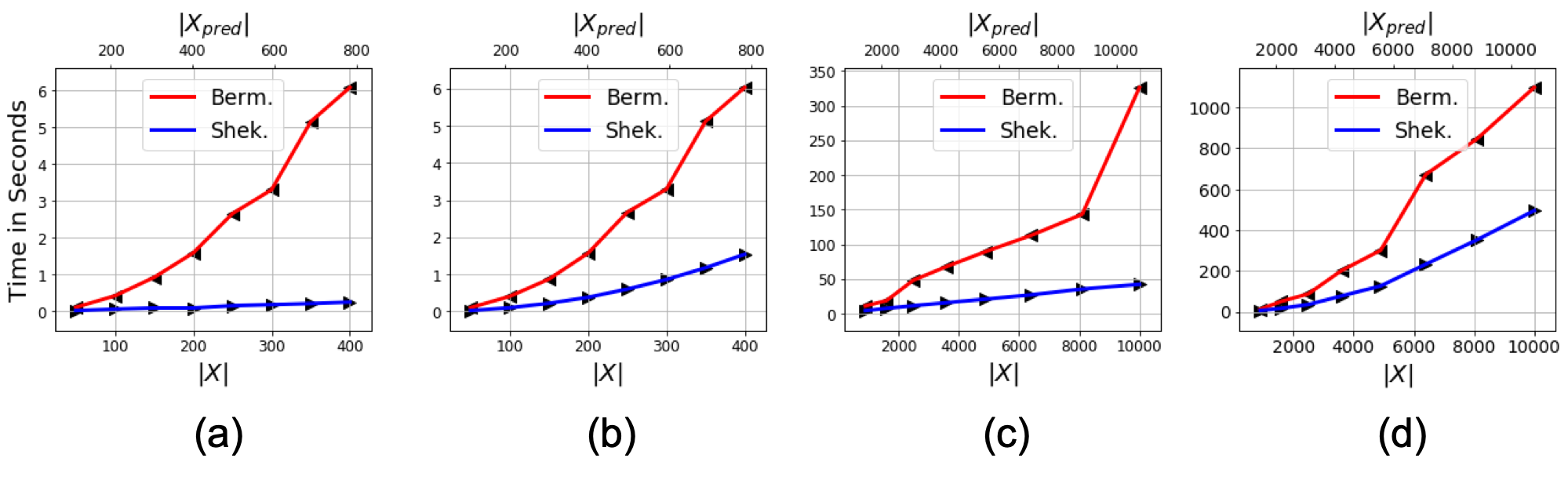}
\caption{Comparison of prediction time in seconds for algorithm 4 from \cite{bermanis2013multiscale} (represented as `Berm') and our hierarchical algorithm \ref{multiscale_algo} (represented as `Shek'). Here the bottom X-axis show the cardinality of the learning dataset with top X-axis representing the cardinality of prediction datapoints}
\end{figure*}

Hence algorithm 4 in \cite{bermanis2013multiscale} has bases from all scales for the final approximation. If we take a closer look at our algorithm (Algorithm \ref{multiscale_algo}), we see that it also samples bases at each scale. However, it just uses the bases at convergence scale as the final bases. The behavior of the algorithm 4 from  \cite{bermanis2013multiscale} is compared with Algorithm \ref{multiscale_algo} in figure \ref{graph12}. Here the decay of error shows similar behavior for both the algorithms. For TF1, Algorithm \ref{multiscale_algo} shows some faster convergence. However, for TF3 and TF4, the multiscale extension algorithm reduces the error to below machine precision faster than Algorithm \ref{multiscale_algo}. Although quite impressive, this doesn't make Algorithm 4 from \cite{bermanis2013multiscale} any more useful because we are already at error levels of $10^{-13}$ at such higher scales. 

With comparative learning behavior, we now move to the comparison of prediction capability. If we think of prediction at a new design point for the case of Algorithm 4 in  \cite{bermanis2013multiscale}, then we will have to keep track of points sampled at each scale from $s= 0$ to $s = s^*$. Once we have that, we can combine the formulated bases of the prediction points points with respect to these points linearly using the projection coordinate at each scale as in equation (\ref{berm}). This is where our algorithm outperforms Algorithm 4 by only just requiring the bases formulated with respect to the sparse representation at the convergence scale. This characteristic of Algorithm 2.1 allows us to talk about sparse representation of the dataset which is not possible with the definition of Algorithm 4 in  \cite{bermanis2013multiscale}. Figure \ref{graph13} shows this more clearly. Here we have measured the time which each of the algorithms take for just prediction, with all the learning assumed to be performed beforehand. Here $|X|$ denotes the size of the data used for learning and $|X_{pred}|$ shows the number design points at which the prediction is to be made. Figure \ref{graph13} visually shows the non-trivial improvement in the prediction latency by our hierarchical approach with repsect to the multiscale algorithm from \cite{bermanis2013multiscale}.

 \begin{figure*}\label{graph14}
\centering
\includegraphics[width=3in]{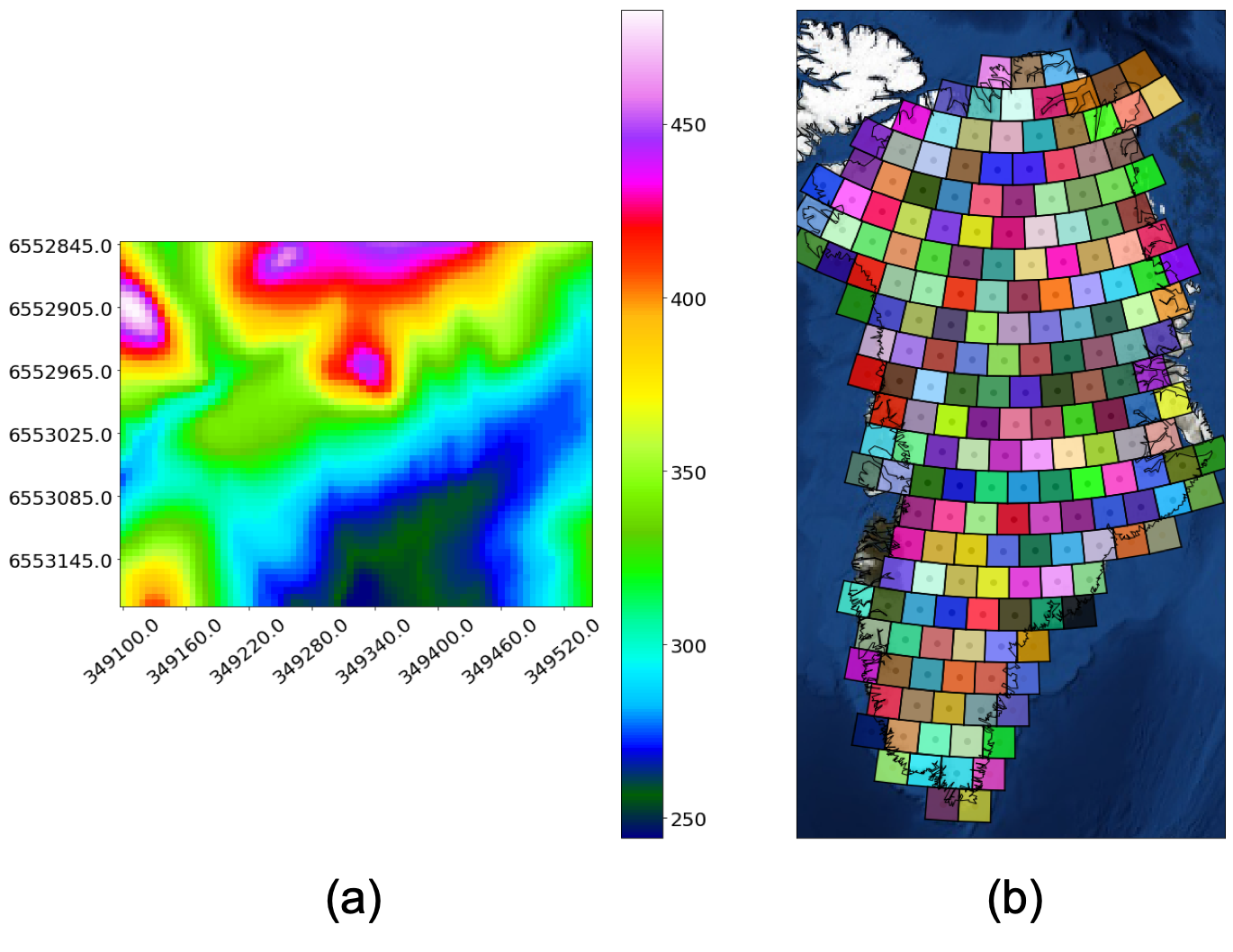}
\caption{(a): Contour plot for Digital Elevation Model (DEM) data; (b): Distribution of the GRACE mascons \cite{luthcke2013antarctica} on Greenland ice sheet showing the spatial resolution of the GRACE data. We use both these datasets for studying the sparse representation and reconstruction capacity of Algorithm \ref{multiscale_algo}}
\end{figure*}

\subsection{Application on Real Data}

In this section we have analyzed the performance of Algorithm \ref{multiscale_algo} on datasets from some practical scenarios. Specifically we are dealing with two different datasets here. The first dataset is spatial in nature (figure \ref{graph14} (a)) where the objective is to learn the sparse representation with the demonstration of capability to reconstruct the data from $D^s_{sparse}$. Here we discuss one additional application, that relates to improving the associated inherent spatial resolution. 

The second category of dataset considered here comes from numerical modeling of gravity measurement changes (figure \ref{graph14} (b)) observed over the Greenland Ice sheet. Here again we analyze the capacity of Algorithm \ref{multiscale_algo} to construct sparse representation of the dataset and reconstruct the dataset from $D^s_{sparse}$.

\subsubsection{Application on Spatial Dataset - generating $D^s_{sparse}$}

For this paper we are considering a particular type of spatial dataset known as Digital Elevation Model (DEM). It is a topographical map of a particular region. The data is arranged on grids with each grid node $(x,y)$ associated with a height measurement. Here X and Y are projection coordinates on a horizontal plane (usually transformed from latitudes and longitudes). Here, we are considering the DEM dataset shown in figure \ref{graph14} (a). The idea in this study is to generate a sparse representation of the DEM data and study the reconstructions produced by these representations as we move up the scale. Figure \ref{graph15} show these results for scale 0,2 and 6. These specific scales were chosen so as to provide an idea of how the surface is evolving towards the starting scale and towards the end. The important thing to note here is that at scale 6, with $D^s_{sparse}$ only consisting of 763 points out of 4350 points, the algorithm was able to generate a reconstruction to the DEM where the prediction error in $\infty-norm$ is just 6.82 (compared to the range of variation observed as $ \sim [250,500]$ in the colorbar in figure \ref{graph14}(a) ).  The compression metric of ($763/4350 \sim 17.5 \%$)  clearly shows the success of the algorithm in generating a sparse representation for the dataset. Here $\infty-norm$ was chosen as an error measure as it upper bounds the error at any individual point and gives an intuitive understanding of the performance of the algorithm.

 \begin{figure*}\label{graph15}
\centering
\includegraphics[width=6.1in]{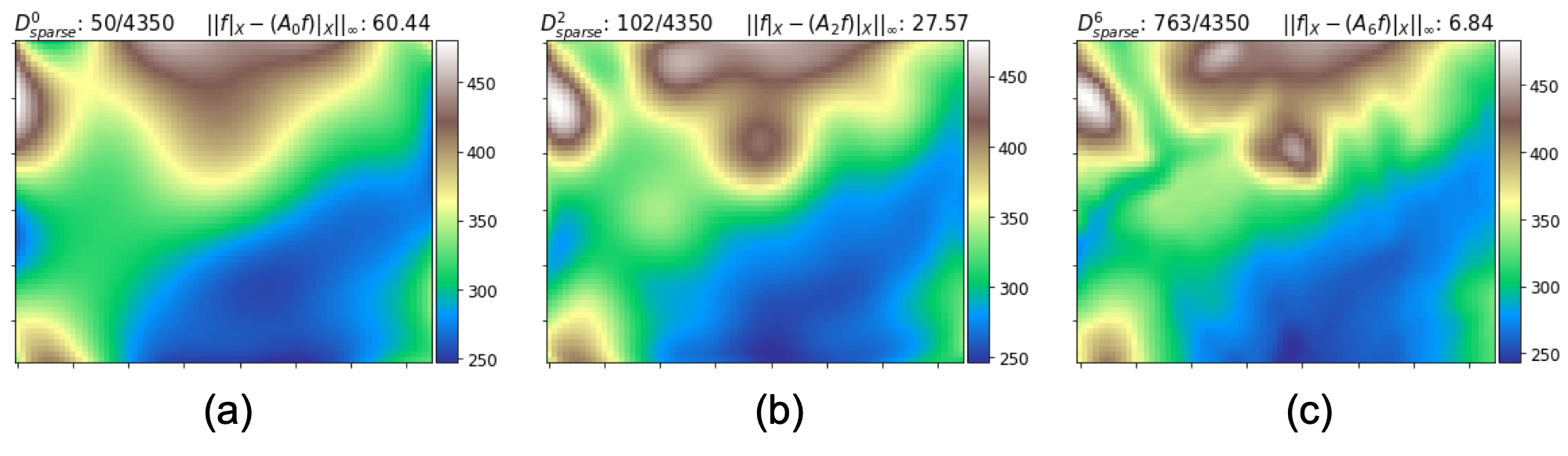}
\caption{Scale dependent reconstructions for the DEM dataset in figure \ref{graph14} (a). Here (a), (b) and (c) show the reconstructed contour plot for scales 0, 2 and 6 respectively. The header on these plots also show the the proportion of dataset selected as the sparse representation along with the $\infty-norm$ of the error in prediction.}
\end{figure*}

\subsubsection{Application on Spatial Dataset - Improving resolution}

For many practical engineering problem like flow simulations, if the modeling region is relatively small, then the resolution of the DEM dataset for the topography plays a very crucial role in the accuracy and stability of the results. The resolution here is defined as the length of side of the square which determines one pixel value. This can also be stated as the length of a cell boundary in the gridded data. This problem of low resolution is also difficult to solve because even if we have DEM data of the same small region from a different source, it usually is from a different time epoch and so there is no guarantee that the topography would not have changed in this time duration. 

The DEM dataset shown in figure \ref{graph14}(a) has a resolution of 60m. However for some particular analysis it might be required to approximate the topography at an even higher resolution. For this reason in this section we have presented the result obtained after interpolating the DEM to a 20m resolution. This result has been presented in figure \ref{graph19}(a). Although here we do not have a proper measure for the performance of the proposed algorithm for this particular task (as the ground truth is unknown), the improved clarity and smoothness of features from top panel to bottom panel in figure \ref{graph19}(a) shows promising nature of results.

 \begin{figure*}\label{graph19}
\centering
\includegraphics[width=6in]{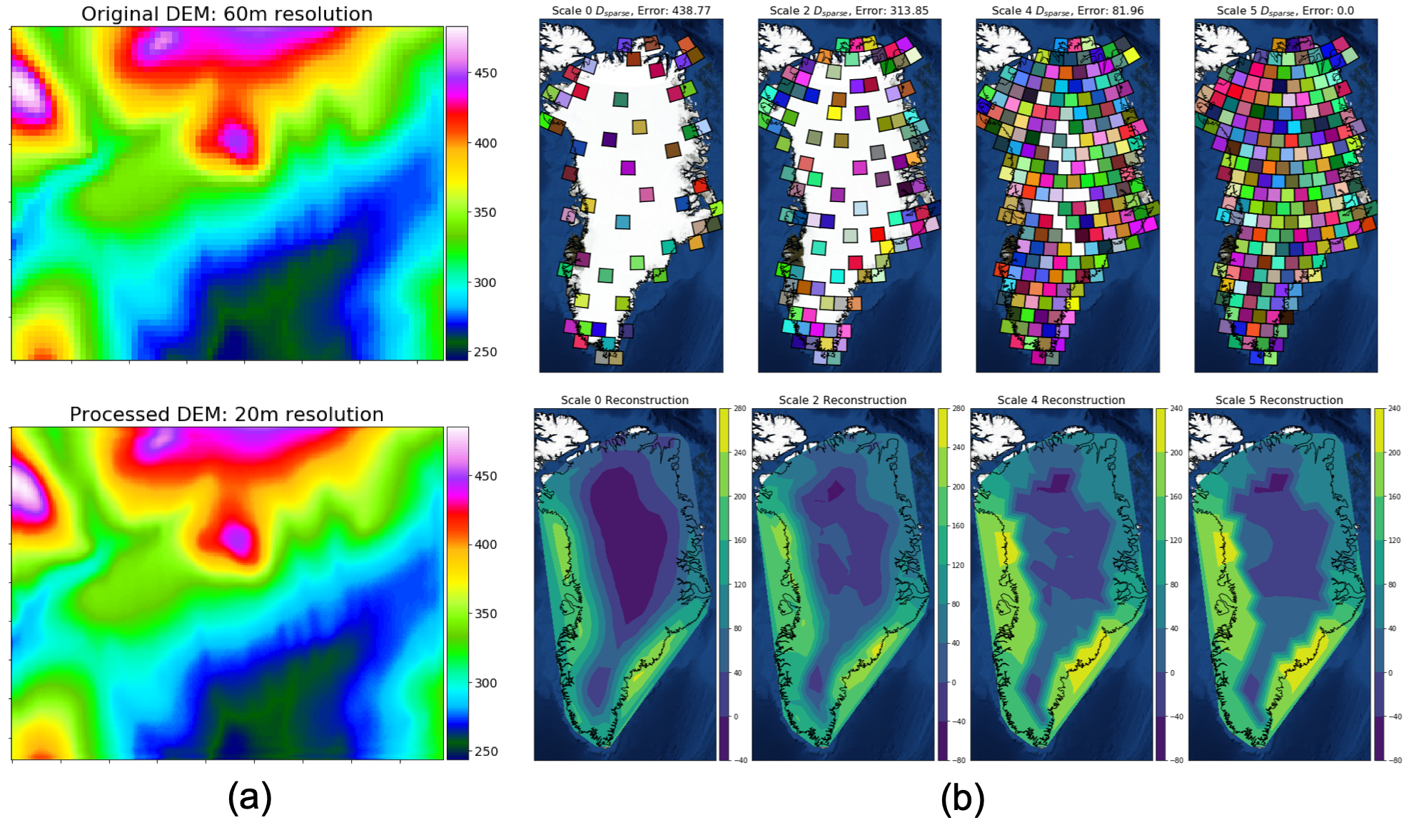}
\caption{(a): Here the top panel shows the DEM dataset in the original resolution (same as figure \ref{graph14}(a)). The bottom panel here shows the result after increasing the resolution to 20m; (b): Application of Algorithm \ref{multiscale_algo} on Greenland Mascons. Here the top row shows the mascons selected at different scales (0,2,4 and 5). The bottom row on the other hand shows the reconstruction of the dataset from the sparse representation selected at each of these respective scales.}
\end{figure*}

\subsection{Application on data from Numerical models}:

In this section we consider the output from a numerical model which has been widely used in the literature \cite{watkins2015improved,alexander2016greenland} for determining the ice mass evolution of Antarctic and Greenland Ice sheets. \cite{luthcke2013antarctica} introduced a iterative strategy for generating the Gravity Recovery and Climate Experiment (GRACE) global solution of equal-area surface mass concentration parcels (also referred to as mascons) in equivalent height of water. Figure \ref{graph14}(b) shows the distribution of these mascons over Greenland ice sheet. The idea is to derive spatially and temporally distributed changes in the mass of land ice at 1 arc degree (approximately 100 km). These mascons are estimated directly from k-band range and range rate (KBRR) data for two co-orbitting satellites roughly 220 km apart. Here each of these approximately $100 \times 100\ Km^2$ square regions have a time series for solution of each mascon in cm. equivalent of water height. For this study we are using the data product $V02.4$ (not corrected for glacial isostatic adjustments). In total, for the entire planet there are 41168 mascons divided broadly among land, ice and water. The time series associated with each of these mascons has 148 entries. We have assumed these observations to be associated with the middle of the mascon solution time window. However, for this study, we just use the spatial aspect of this dataset (just consider observations only at t=0 for all mascons) and showcase the capability of our approach to generate a sparse representation for this dataset. Figure \ref{graph19}(b) shows the sparse representation of the mascons at scale 0, 2, 4 and 5. In the bottom row, the corresponding reconstruction from the sparse representation has been shown as well. The main thing to notice here is that even at scale 0, with only a portion of the original mascons, our approach was able to learn the behavior of the dataset (this is evident from the comparison of reconstructions from the different scales)

\section{Conclusion}

In this paper, we have introduced a hierarchical method for learning a sequence of sparse representations for a large dataset. The hierarchy comes from  approximation spaces defined based on a scale parameter. Principally, the proposed  approach has been shown to be useful for data reduction applications coupled with learning a model for representing the data.  The paper provides  analysis  that explains and studies the theoretical properties of the proposed approach. We derive bounds for stability, convergence and behavior of error functionals. In the results section, we have shown the performance of the approach as data reduction mechanism on both synthetic and real datasets (geo-spatial and numerical modeling datasets). The sparse model generated by the presented approach is also shown to efficiently reconstruct the data, while minimizing error in prediction.

Though the results shown in this paper depict the efficiency of the approach on a variety of datasets and settings, there are several areas in which the presented algorithm can be improved. Firstly, the implementation of the algorithm can be made more efficient by either optimizing the operations in the algorithm or by handling chunks of data at a time. Secondly, generation of sparse representation for noisy datasets poses another set of unique challenges which will be addressed in a companion paper under preparation. For this case, properly capturing the uncertainty in the generated sparse approximations becomes very crucial. Finally the hyperparameters like $P$ in (Algorithm \ref{multiscale_algo}) can be explored and studied further for making the approach behave better. This can be done by even utilizing additional information which is not directly observed but is inherently known (like the physics of a system). Hence, in essence the work presented  addresses the need for efficient learning methods for large datasets and opens up new interesting approaches.

\bibliographystyle{siamplain}
\bibliography{SIMOD_pshekhar}

\end{document}